\documentclass{article}
\usepackage[margin=1.2in]{geometry}
\usepackage{setspace}
\usepackage{amssymb}

\usepackage{multirow}
\usepackage{colortbl}
\usepackage{threeparttable}

\usepackage{microtype}
\usepackage{graphicx}
\usepackage{subfigure}
\usepackage{booktabs}
\usepackage{hyperref}

\usepackage[authoryear]{natbib}

\usepackage{amsmath}
\usepackage{amssymb}
\usepackage{mathtools}
\usepackage{amsthm}
\usepackage{amsfonts}

\usepackage[capitalize,noabbrev]{cleveref}

\theoremstyle{plain}
\newtheorem{theorem}{Theorem}[section]
\newtheorem{proposition}[theorem]{Proposition}
\newtheorem{lemma}[theorem]{Lemma}
\newtheorem{corollary}[theorem]{Corollary}
\theoremstyle{definition}

\newtheorem{remark}{Remark}

\usepackage{todonotes}

\usepackage[algo2e,lined,boxed,commentsnumbered,ruled]{algorithm2e}
\usepackage{setspace}
\RestyleAlgo{ruled}

\usepackage{accents}

\usepackage{enumitem}
\usepackage{nicefrac}       
\usepackage{textcomp}

\usepackage{comment}

\definecolor{mygray}{gray}{0.8}

\allowdisplaybreaks

\DeclareMathOperator*{\argmax}{argmax}

\newcommand{\numberthis}{\addtocounter{equation}{1}\tag{\theequation}}

\newcommand{\lrset}[1]{\left\{ #1 \right\}}
\newcommand{\lrp}[1]{\left( #1 \right)}
\newcommand{\lrs}[1]{\left[ #1 \right]}

\newcommand{\KLinf}{\operatorname{KL_{inf}}}

\newcommand{\KL}{\operatorname{KL}}

\newcommand{\betau}{\beta_u}
\newcommand{\betal}{\beta_l}

\newcommand{\Exp}[2]{\mathbb{E}_{#1}\lrs{#2}}

\newcommand{\N}{\mathbb{N}}

\newcommand{\calE}{\mathcal E}

\title{
 Cover meets Robbins while Betting on Bounded Data:\\ $\ln n$  Regret and Almost Sure $\ln\ln n$ Regret
}
\date{}
\author{Shubhada Agrawal$^1$, Aaditya Ramdas$^2$\\
$^1$Indian Institute of Science, $^2$Stanford University\\
\texttt{shubhada@iisc.ac.in, aramdas@stonford.edu}\\
\smallskip \\
\today
}

\begin{document}
\maketitle


\begin{abstract}
Consider betting against a sequence of $[0,1]$-valued observations using bets that are fair if the data have a conditional mean $m_0 \in (0,1)$. Cover's universal portfolio algorithm achieves a worst-case regret of $O(\ln n)$  compared to the best constant bet in hindsight, and this logarithmic order is unimprovable against adversarial data. We construct an explicit full-support Robbins mixture that combines insights from Robbins and Cover. It retains $O(\ln n)$ regret on every path while also eventually satisfying $R_n\le(1+o(1))\ln\ln V_n$ on every path with bounded mixture wealth and $V_n\to\infty$, where $V_n$ is an intrinsic variance process. It also attains the optimal asymptotic wealth growth rate under iid alternatives. Under the conditional-mean claim and almost-sure divergence of $V_n$, the sharper regret bound holds almost surely, and its leading constant $1$ is unimprovable. We also give explicit finite-time, path-wise regret bounds. A simple hedging principle shows how splitting initial wealth between two strategies achieves a best-of-both-worlds guarantee by retaining the smaller of their regrets and the larger of their growth rates. We contrast our results to those for unbounded data in~\citet{agrawal2025regret}. Their worst-case regret is unbounded. However, hedging yields almost-sure iterated-logarithmic regret together with optimal growth. Finally, our wealth process witnesses a sharp game-theoretic upper law of iterated logarithm, analogous to~\cite{shafer2005probability}.
\end{abstract}

\section{Introduction}\label{sec:intro}
Sequential betting provides a way to evaluate predictions as data arrive. Suppose that a forecaster repeatedly claims that the next bounded observation $X_n\in[0,1]$ has conditional mean $m_0\in(0,1)$ given the past observations. For example, $X_n$ may indicate whether the $n$-th visitor clicks on an advertisement, with $m_0$ a claimed click-through rate. A bettor may challenge this claim by placing wealth on observations that fall systematically above or below $m_0$. The available bets are fair under the forecaster's claim, but the bettor's wealth may grow when the claim is incorrect. Consequently, the resulting wealth process has two complementary interpretations: it is the payoff of an online betting strategy on an arbitrary data sequence, and, under stochastic assumptions, it provides sequential evidence against the forecaster's conditional-mean claim. Indeed, if the claim holds, the wealth is a non-negative supermartingale starting from one. So Ville's inequality \citep{ville1939etude} bounds the probability that it ever reaches $1/\alpha$ by $\alpha$, for every $\alpha \in (0,1)$. Large wealth therefore provides  strong evidence against the claim. Appendix~\ref{app:bettinggame} presents the formal betting protocol.

The bettor does not know in advance which fixed betting fraction will perform best. A natural remedy is a \emph{mixture} strategy, which splits its initial wealth across admissible fixed betting fractions according to a prior and lets each portion keep betting its own fraction. We compare such a strategy with an oracle that chooses the best fixed betting fraction in hindsight. The oracle's log-wealth minus the bettor's log-wealth is the \emph{regret}. Regret of $O(\ln n)$ after $n$ rounds means that the bettor's wealth is within a polynomial factor of the oracle's, and $O(\ln\ln n)$ means it is within a polylogarithmic factor. Note that this is the coin-betting formulation underlying universal portfolios and parameter-free online learning \citep{cover1991universal,orabona2016coin}. 

We seek two kinds of guarantees: \emph{worst-case} guarantees that hold on every sequence, including adversarial ones, and sharper guarantees on \emph{typical paths}, namely those on which the bettor's wealth stays bounded. For the latter, we distinguish finite-time high-probability guarantees from eventual almost-sure guarantees. Both kinds of guarantees are deterministic statements about individual paths; stochastic assumptions enter only to show that typical paths form a large set. We study the Cover-style \citep{cover1991universal} and Robbins-style \citep{robbins1970statistical} mixtures in both regimes.

A uniform Cover-style mixture spreads its prior mass across the full interval of admissible bets and achieves order-optimal worst-case regret of $O(\ln n)$ after $n$ observations. When the data are independent and identically distributed (iid) from a distribution with mean different from $m_0$, it also achieves an optimal asymptotic wealth growth rate (long-run rate at which log-wealth accumulates). Its standard regret guarantee, however, remains $O(\ln n)$. A Robbins-style mixture instead concentrates prior mass near zero. Since the hindsight-optimal betting fraction shrinks towards zero on typical paths as data accumulate, this allows it to achieve $O(\ln\ln n)$ regret on such paths. The classical version used by \citet{orabona2023tight} restricts its support to a subinterval of admissible bets. Relative to the full set of admissible bets, this restriction leads to linear regret on some paths and a suboptimal asymptotic wealth growth rate under certain iid\ alternatives whose optimal fixed bet lies outside the subinterval.

These behaviours suggest that improving regret on typical paths may require sacrificing performance on other paths or asymptotic wealth growth or both. A related regret-growth trade-off was also highlighted by \citet{agrawal2025regret} in the unbounded-data setting. Our results show that these apparent trade-offs can be avoided. For bounded data, we construct and analyse an explicit modified Robbins mixture that places substantial prior mass near zero while retaining support over the entire admissible bet interval. It achieves $O(\ln n)$ regret on every path and, when the conditional-mean claim holds and the intrinsic variance process $V_n\to\infty$ almost surely, eventually achieves $O(\ln\ln V_n)$ regret almost surely. Here, $V_n:=\sum_{i\le n}(X_i-m_0)^2\le n$ is the accumulated squared deviations from \(m_0\), and is formally introduced later. The mixture also achieves the optimal asymptotic wealth growth rate under suitable stochastic alternatives.

This construction relies on boundedness and does not carry over directly to the unbounded-data setting of \citet{agrawal2025regret}, where the betting parameter ranges over $\mathbb{R}$. Nevertheless, a simple hedging principle combines favourable regret and growth guarantees in both settings: splitting the initial wealth between two strategies retains, up to an additive constant, the smaller of their regrets (measured against the same comparator) on every path and, against iid\ alternatives, the larger of their growth rates. For sub-Gaussian data, this yields $O(\ln\ln n)$ typical-path regret together with optimal growth, although, unlike in the bounded case, without a worst-case regret guarantee. We now formally introduce the setup.

\paragraph{Setup.} Let $X_1,X_2,\dots$ take values in $[0,1]$. Fix $m_0\in(0,1)$, and for $\lambda\in I_{m_0} := [-1/m_0, 1/(1-m_0)]$, define
\[
W_n(\lambda):=\prod_{i=1}^{n}\bigl(1-\lambda(X_i-m_0)\bigr), \qquad\text{with} \qquad W_0(\lambda)=1.
\]
The interval $I_{m_0}$ is exactly the set of bets for which every single-round multiplier is non-negative for every observation in $[0,1]$, i.e., $1-\lambda(y-m_0) \ge 0$ for every $y\in[0,1]$. If the forecaster's conditional mean claim holds, then $\mathbb{E}[1-\lambda(X_t-m_0)\mid X_1,\dots,X_{t-1}]=1$, so every such bet is fair. The sign of $\lambda$ selects which side of $m_0$ the bettor favours, and $|\lambda|$ controls the exposure. Hence, $W_n(\lambda)$ is the wealth obtained by using the same betting fraction $\lambda$ in every round, starting from a unit wealth. 

Rather than selecting one $\lambda$ in advance, a mixture strategy distributes its initial unit wealth according to a prior density $\pi$ on $I_{m_0}$:
\[
W_n:=\int_{I_{m_0}}W_n(\lambda)\pi(\lambda)d\lambda,\qquad\text{with}\qquad W_0=1. 
\numberthis\label{eq:mixtureprocess}
\]
The hindsight-optimal oracle wealth and the mixture's regret are respectively given by 
\[
W_n^*:=\sup_{\lambda\in I_{m_0}}W_n(\lambda),
\qquad\text{and}\qquad
R_n:=\ln W_n^*-\ln W_n.
\]
Thus, regret is precisely the price of not knowing the fixed best-in-hindsight bet in advance. 

Finally, for $\alpha\in(0,1)$, define the following \emph{Ville event} as the set of paths along which wealth remains controlled: 
\[
\calE_\alpha:=\left\{\sup_{n\ge1}~W_n\le\frac1\alpha\right\}.
\numberthis\label{eq:ealpha}
\]
Also define the bounded-wealth event
$\calE_0:=\{\sup_{n\ge1}W_n<\infty\}.$ Unlike $\calE_\alpha$, this event allows the bound on wealth to depend on the path. 

Note that all quantities defined above are deterministic functions of the observed sequence, so our path-wise results require no probability model. If the data are generated under a law $P$ for which $(W_n)$ is a non-negative supermartingale, Ville's inequality gives $P(\calE_\alpha)\ge1-\alpha$, while the supermartingale convergence theorem gives $P(\calE_0)=1$. Thus, ``every path'' refers to a guarantee over arbitrary sequences, whereas ``almost every path'' means probability one under a specified stochastic model. Paths in $\calE_\alpha$ and $\calE_0$ are the \emph{typical paths} referred to above.

\paragraph{Contributions and roadmap.} Theorem~\ref{th:ville_gaussianmixture} recalls the $O(\ln n)$ path-wise regret guarantee for the Cover-style uniform mixture, which serves as a benchmark for the contributions below. 
\begin{enumerate}
\item \textbf{A single mixture with simultaneous guarantees.}
In Section~\ref{sec:regretbound}, we construct an explicit full-support Robbins mixture with $O(\ln n)$ regret on every bounded sequence and optimal asymptotic wealth growth under iid alternatives. Its density dominates a positive multiple of the uniform prior, transferring Cover's worst-case protection while retaining substantial mass near zero. Theorems~\ref{th:bddregretbound} and~\ref{th:lilasregretV} give explicit finite-time regret bounds of iterated-logarithmic order on the Ville event $\calE_\alpha$ and a stronger eventual bound
$R_n\le(1+o(1))\ln\ln V_n$
on every path with bounded wealth and $V_n\to\infty$; the leading constant $1$ cannot be improved (Remark~\ref{rem:regret-sharpness}). The main technical work lies in these bounds, since the prior is asymmetric and the hindsight-optimal bet may lie on the boundary of $I_{m_0}$. Under the conditional-mean claim and almost-sure divergence of $V_n$, this eventual guarantee holds almost surely.

\item \textbf{Explicit regret analysis of restricted mixtures.}
Appendix~\ref{app:OJ} makes explicit the regret guarantee implicit in \citet{orabona2023tight} and shows that their restricted-support mixture can incur linear regret against the full comparator class $I_{m_0}$ (Proposition~\ref{prop:lilregretoje_alpha} and Remark~\ref{rem:linearOJregret}). Our modified prior removes this failure while preserving the typical-path improvement.

\item \textbf{A general hedging principle.}
While the explicit construction in Section~\ref{sec:regretbound} achieves both guarantees simultaneously for bounded data, it does not extend directly to unbounded data. Instead, Proposition~\ref{prop:BOB} in Section~\ref{sec:BOB} shows that a fixed convex combination with strictly positive weights inherits the smaller component regret, up to an additive constant, and an asymptotic growth rate at least as large as either component's. These advantages may come from different components. Applied to the sub-Gaussian mixtures of \citet{agrawal2025regret}, it resolves the regret-growth trade-off discussed there, yielding an iterated-logarithmic typical-path regret together with optimal growth under alternatives, although worst-case regret remains unbounded in that setting.

\item \textbf{An explicit game-theoretic Law of Iterated Logarithm (LIL) witness.} Our regret bound in Theorem~\ref{th:lilasregretV} gives an explicit game-theoretic witness for the self-normalised upper LIL: whenever a path with $V_n\to\infty$ violates the relevant LIL bound ($\limsup_n |S_n|/\sqrt{2V_n\ln\ln V_n}\le 1$, where $S_n:=\sum_{i\le n}(X_i-m_0)$), the proposed wealth process becomes unbounded, with $\limsup_n W_n = \infty$. This implication holds without a probability model and links our online learning framework to game-theoretic probability. 
\end{enumerate}

The same wealth process therefore serves three complementary roles: it competes with the best fixed bet on arbitrary bounded sequences, yields anytime-valid tests of the conditional-mean claim, and witnesses the self-normalised upper LIL. Table~\ref{tab:comparison} in Appendix~\ref{app:table} summarises the regret and growth guarantees of all strategies discussed, and Section~\ref{sec:prelims} reviews background and related work.

\section{Preliminaries}\label{sec:prelims}
This section collects the necessary background and reviews related work. 
\subsection{Background}\label{sec:bg}
We recall nonnegative (super)martingales and their interpretation as wealth processes, the strong law of large numbers (SLLN) and the law of iterated logarithm (LIL) together with their self-normalised versions, and the $\KLinf$ functional that governs the growth rates of the wealth processes studied in this work. When working with stochastic data, we will work on a filtered measurable space $(\Omega, \mathcal F)$, where $\mathcal F = (\mathcal F_n)_{n\ge 0}$ is a filtration. Let $\Pi$ be a family of probability measures defined on $(\Omega, \mathcal F)$. In this work, since the data are assumed to be bounded, $\Omega = [0,1]^\infty$. Further, at each time $n$, we observe a single observation $X_n$ so that $\mathcal F_0 = \{\emptyset, \Omega\}$ and for $n\ge 1$,  $\mathcal F_n = \sigma(X_1, \dots, X_n)$. A stochastic process $M=(M_n)_{n\ge0}$ is called adapted if $M_n$ is $\mathcal F_n$-measurable for every $n$. 

\paragraph{Non-negative martingales (NM), supermartingales (NSM), and wealth processes.} For a probability measure $P\in\Pi$, the process $M$ is called a $P$-supermartingale if $\mathbb E_P[|M_n|]<\infty$ for all $n$, and $\mathbb E_P[M_n \mid \mathcal F_{n-1}] \le M_{n-1}$, {$P$-almost surely, for all $n\ge 1$.} If instead, equality holds for every $n$, then $M$ is a $P$-martingale. Given a class of measures $\mathcal P \subset \Pi$, we say that $M$ is a $\mathcal P$-(super)martingale if it is a $P$-(super)martingale for each $P\in\mathcal P$. When the underlying measure is clear from context, we simply refer to $M$ as a (super)martingale. NMs and NSMs admit a natural interpretation in terms of wealth processes in sequential betting games, which we discuss next. 

A wealth process $(W_n)_{n\ge0}$ is an adapted process with $W_0=1$ and $W_n\ge0$ for all $n$. One may view $W_n$ as the capital of a bettor who starts with unit wealth and repeatedly stakes a fraction of their current capital on future observations. In a fair (or conservative) game, the resulting wealth process is a NM (or NSM) with respect to the data filtration, and Ville’s inequality \citep{ville1939etude} provides a time-uniform control: $P[\sup_n W_n \ge {1}/{\alpha}] \le \alpha$ for every $\alpha \in (0,1)$. In particular, for every $\lambda\in I_{m_0}$, the process $(W_n(\lambda))_{n\ge0}$, and hence every mixture $(W_n)_{n\ge0}$ in~\eqref{eq:mixtureprocess}, is an NM under any law for which $\mathbb E[X_n\mid\mathcal F_{n-1}]=m_0$ for all $n\ge1$.

Conversely, any NM $(M_n)_{n\ge0}$ with $M_0=1$ corresponds to a wealth process generated by a sequence of multiplicative bets in a fair game. To see this, let the per-round multiplier at time $n$ be given by the conditional e-variable \citep[\S1]{ramdas2024hypothesis}
$E_n := ({M_n}/{M_{n-1}})\mathbf 1\{M_{n-1}>0\}$, with $(0/0:=0)$, so that $M_n = \prod_{i=1}^n E_i$. The martingale property of $M$ is exactly the fairness of each bet: $\mathbb E[E_n\mid\mathcal F_{n-1}]=1$ whenever $M_{n-1}>0$. This viewpoint plays a central role in game-theoretic probability and statistics, where NMs are often interpreted as capital processes or test martingale. 

\paragraph{Strong law of large numbers (SLLN) and law of iterated logarithm (LIL).} When $X_i$ are iid from some distribution with support in $[0,1]$ and mean $m_0$, the SLLN guarantees that $\lim_n\tfrac{1}{n}\sum\nolimits_{i=1}^n (X_i - m_0) = 0$, almost surely, while the LIL refines this as: $$\limsup\limits_{n\to\infty} \frac{|\sum\nolimits_{i=1}^n (X_i-m_0)|}{\sigma\sqrt{2n\ln\ln n}} = 1, \quad \text{almost surely},$$ where $\sigma^2 > 0$ is the variance of the unknown distribution. With $S_n:=\sum_{i=1}^n(X_i-m_0)$ and the intrinsic variance process $V_n:=\sum_{i=1}^n(X_i-m_0)^2$, these results also hold in self-normalised form: $S_n/V_n\to0$ and $\limsup_{n}|S_n|/\sqrt{2V_n\ln\ln V_n}=1$, almost surely. 
Theorem~\ref{th:lilasregretV} recovers the upper halves of these statements path-wise, on every path on which our wealth process stays bounded and $V_n\to\infty$. These limit theorems explain the target rate for regret. For small bets, $\ln(1-\lambda(X_i-m_0))\approx-\lambda(X_i-m_0)-\tfrac{\lambda^2}{2}(X_i-m_0)^2$, so $\ln W_n^*\approx S_n^2/(2V_n)$. Under the conditional-mean claim, the LIL therefore makes the oracle's log-wealth itself of order $\ln\ln V_n$ infinitely often. Hence, $\ln\ln V_n$ is the natural, and unimprovable, scale for regret on typical paths. We revisit this connection later, when we present our regret bounds in Section~\ref{sec:regretbound}. 

\paragraph{KL-inf and its properties.} A specific information projection function (infimum of Kullback Leibler or $\KL$ divergences) plays a crucial role in the growth rate analysis of the different wealth processes in this work. We introduce it now. For probability measures $Q$ and $P$, the Kullback Leibler or $\KL$ divergence is $\KL(Q,P)=\mathbb E_Q[\ln\frac{dQ}{dP}(X)]$ if $Q$ is absolutely continuous with respect to $P$, and $+\infty$ otherwise. Given any collection $\mathcal P$ of probability measures, and a probability measure $Q$, define $\KLinf(Q,\mathcal P) := \inf\{\KL(Q,P): P \in \mathcal P\}.$ It measures the divergence from \(Q\) to the closest distribution in $\mathcal P$. The $\KLinf$ optimization problem (dual formulation and its structural and topological properties) has been studied extensively for specific classes of $\mathcal P$ in the stochastic multi-armed bandit literature. We refer the reader to \cite[\S4]{thesis} for an exposition. 

Let $\mathcal P[0,1]$ denote the collection of probability measures with support in $[0,1]$. For $Q \in \mathcal P[0,1]$ and $m_0 \in (0,1)$, define
\begin{align*} 
    \KLinf(Q,m_0) &:= \inf\limits_{\substack{P' \in \mathcal P[0,1]; ~ \mathbb{E}_{P'}[X] = m_0}}~\KL(Q,P') = \max\limits_{\lambda \in I_{m_0}}~\Exp{Q}{\ln(1-\lambda (X-m_0) )}, \numberthis\label{eq:klinf}
\end{align*}
where the last equality follows from \citet{honda2010asymptotically}. Observe that if $m_0 = \mathbb{E}_Q[X]$, then $\KLinf(Q, m_0) = 0$. This is because $Q$ is itself a feasible solution for the optimization defining $\KLinf$. Further, in this case, $\lambda = 0$ achieves the maximum. Next, \citet{honda2010asymptotically} show that for $m_0 \in (0,1)$, the function $\KLinf(\cdot, m_0)$ is continuous in the weak topology on $\mathcal P[0,1]$. In particular, if $\hat{Q}_n$ denotes the empirical distribution obtained using $n$ iid samples from $Q\in\mathcal P[0,1]$, then for any $x\in (0,1)$, $\KLinf(\hat{Q}_n, x) \to \KLinf(Q,x)$ almost surely as $n\to \infty$ (since $\hat{Q}_n \to Q$ almost surely, in the weak topology). 

Applying the dual formula to $\hat Q_n$ gives $\ln W_n^*=n\,\KLinf(\hat Q_n,m_0)$. Hence, when the data are iid from $Q$ with mean different from $m_0$, the oracle's log-wealth grows at rate $\KLinf(Q,m_0)$, which is the best rate achievable in this game. This is the \emph{optimal growth rate} referred to in the introduction.

\subsection{Literature Review}\label{sec:literature}
We now briefly review the relevant literature. A full treatment of the connections among these works, which span multiple sub-fields, is left for future work. We refer the reader to \citet{agrawal2025regret} for a broader discussion. 

The use of mixture wealth processes for constructing confidence sequences (i.e., sequences of confidence intervals) for the mean dates back at least to \citet{darling1967confidence}, who studied parametric settings. The specific non-parametric mixture wealth process in~\eqref{eq:mixtureprocess} has also been explored in prior work \citep{agrawal2020optimal,waudby2024estimating,orabona2023tight}. In particular,  \citet{orabona2023tight} employ this wealth process with a mixing distribution originally introduced by  \citet{robbins1970statistical} to construct a confidence sequence (CS) for the mean of a distribution supported on $[0,1]$ with width $O(\sqrt{(\ln\ln n)/n})$. While there exist many other methods to derive CSs with those widths~(for example, \citet{howard2021time}, \citet{kaufmann_mm_21}, and~\citet{waudby2024estimating}), their proof is distinctive in that it proceeds by deriving an implicit, data-dependent regret bound for the wealth process against a smaller comparator class than the one considered here, which is shown to be $O(\ln\ln n)$ on a certain high probability event (termed ``Ville'' event in this paper). In fact, the link between regret and inference is direct: the CS excludes $m_0$ once $\ln W_n^*\ge\ln(1/\alpha)+R_n$ so that $\ln W_n\ge\ln 1/\alpha$. Next, $\ln W_n^* = n  \KLinf(\hat{Q}_n, m_0) \ge n \KL(\operatorname{Ber}(\hat{m}_n), \operatorname{Ber}(m_0)) \ge 2 n(\hat m_n-m_0)^2$, where $\hat m_n$ is the empirical mean, and the first inequality is the data-processing inequality. A regret $R_n$ yields a CS width of order $\sqrt{(\ln(1/\alpha)+R_n)/n}$, and $O(\ln\ln n)$ regret yields the LIL-optimal width.

Following this line of work, we make explicit what is implicit in \citet{orabona2023tight} and extend it to the full comparator class. While inspired by their approach, our results differ in three key aspects: (i) we provide explicit regret guarantees that are not directly available from their analysis, (ii) we compete with the full class of admissible bets rather than a restricted subset considered by the authors, and (iii) most importantly, our construction achieves $O(\ln\ln n)$ regret on typical paths together with $O(\ln n)$ worst-case regret and the optimal growth rate, whereas their mixture, designed for typical-path behaviour, can incur linear regret against the full comparator class.

Next, in a recent work, \citet{agrawal2025regret} studied regret of mixture wealth processes in the sub-Gaussian setting, where the data are not assumed to be bounded. There, a Gaussian mixture has $O(\ln n)$ regret on typical paths and the optimal growth rate under alternatives, whereas a Robbins mixture has the sharper $O(\ln\ln n)$ typical-path regret but may incur linear regret on other paths and has a suboptimal growth rate. Applying our hedging principle (Proposition~\ref{prop:BOB}) to these two mixtures retains the sharper typical-path regret and the optimal growth rate. However, since neither component has $O(\ln n)$ regret on every unbounded sequence, the hedge does not yield a worst-case guarantee in that setting. Boundedness is what enables the stronger result here: because the admissible bet interval $I_{m_0}$ is compact, Cover's uniform mixture has $O(\ln n)$ regret on every path, and our full-support modified Robbins mixture inherits this protection while retaining $O(\ln\ln V_n)$ typical-path regret and the optimal growth rate. Thus, for bounded data, a single mixture suffices, whereas Proposition~\ref{prop:BOB} is the broader mechanism that applies in both settings. See Section~\ref{sec:tradeoffs} and Table~\ref{tab:comparison} for a comparison.

In the online learning and information theory literature, \citet{cover1991universal, cover2002universal} establish deterministic regret guarantees for mixture strategies in sequential prediction and portfolio selection. In particular, they prove an $O(\ln n)$ path-wise regret bound against all sequences of observations from $[0,1]$. Subsequent works propose alternative mixture wealth processes that improve the constant terms in the path-wise regret bound of $O(\ln n)$. 

A related line of work studies adaptation to unknown learning rates in online learning. In prediction with expert advice, \citet{koolen2015second} derive deterministic second-order regret bounds with leading term $O(\sqrt{V'_n(\mathrm{complexity}+C_n)})$, alongside additive complexity terms. Here, $V'_n$ is the cumulative squared excess loss, the complexity term is a comparator penalty (e.g., $\ln K$ for $K$ experts), and $C_n$ is the cost of adapting to the unknown learning rate. A prior concentrated near zero yields $C_n=O(\ln\ln V'_n)$ as $V'_n\to\infty$. Their prior is analogous to the heavy-near-zero Robbins prior, with the learning-rate parameter in their setting playing a role similar to the betting fraction in our setup. The iterated-logarithmic term enters their second-order regret bound as a learning-rate adaptation penalty $(C_n)$ inside the square root. In our setting, an iterated-logarithmic bound controls the entire log-wealth regret on typical paths.


Finally, \citet[Chapter~5]{shafer2005probability} establish a game-theoretic version of the LIL that holds path-wise and applies to general betting games, including the bounded betting game considered here. Their result shows that any path violating the LIL admits a betting strategy whose capital diverges. Their focus is on the existence of such strategies rather than on finite-time guarantees or a quantitative comparison with the best fixed bet in hindsight. In contrast, we construct an explicit mixture wealth process and derive path-wise regret bounds relative to the hindsight-optimal strategy that hold for every $n$. In our framework, the upper half of the self-normalised LIL emerges as a consequence of controlling this regret; we do not address the lower half.

\section{Warm Up: Logarithmic  Regret}\label{sec:logregret}
We begin with a simple setting when the prior $\pi$ on $\lambda$ is a uniform distribution on the interval $I_{m_0}$. Then, the mixture (wealth) process from~\eqref{eq:mixtureprocess} becomes
\[
W_n = {m_0(1-m_0)} \int_{I_{m_0}}~ \prod\limits_{i=1}^n (1-\lambda(X_i - m_0))d\lambda.
\numberthis\label{eq:mixV}\]

Let $V_n := \sum_{i=1}^n (X_i - m_0)^2$, and $S_n := \sum_{i=1}^n (X_i - m_0)$. We initialize as $V_0 = 0$ and $S_0 = 0$. Note that $S_n = 0$, $R_n = 0$ whenever $V_n = 0$. 
\begin{theorem}\label{th:ville_gaussianmixture}
    For $W_n$ defined in~\eqref{eq:mixV}, and for all $n\ge 0$, 
    \begin{align*} R_n 
    &:= \max_{\lambda\in I_{m_0}}~\sum\limits_{i=1}^n \ln(1-\lambda(X_i - m_0)) - \ln W_n \le \ln(1+n)+ 1. \numberthis\label{eq:regretbound} 
    \end{align*}
    In particular, on the event
    \(\calE_0 := \lrset{  \limsup\nolimits_n W_n < \infty},
    \) dividing the above inequality by $n$ and taking limit as $n \to \infty$, we get  
    $|S_n|/n \to 0$ path-wise. Contrapositively, if there is a path on which $|S_n| /n \to 0$ does not hold, then $\limsup_n W_n = \infty$ on that path. Finally, if  data are drawn from a distribution $P$ with support in $[0,1]$ such that $(W_n)_{n\ge 1}$ is a non-negative supermartingale, then  $\calE_0$ is a measure $1$ event. 
\end{theorem}
First, note that unlike in the sub-Gaussian setting, here the regret is $O(\ln n)$ \emph{on all paths}. Next, Theorem~\ref{th:ville_gaussianmixture} proves a deterministic statement of the form: either 
$|S_n|/n \to 0$ as $n\to \infty$, or $\limsup_n W_n = \infty$. In particular, when the data are stochastic, it provides an explicit \emph{witness} supermartingale for the violation of the strong law, meaning that if the SLLN is violated on some path, $\limsup_n W_n = \infty$ on that path. As in the sub-Gaussian setting, here we have a single supermartingale $(W_n)_{n\ge 1}$ that serves as a witness for the entire class of distributions $\mathcal P$ with support in $[0,1]$ and conditional mean $m_0$.

While Theorem~\ref{th:ville_gaussianmixture} gives an $O(\ln n)$ upper bound on the regret of Cover-style uniform mixture, the lemma below shows that it does not improve on typical paths. We refer the reader to Appendix~\ref{app:proof_warmuplogTregret} for a complete proof of Theorem~\ref{th:ville_gaussianmixture} and Lemma~\ref{lem:cover-lower}. 

\begin{lemma}\label{lem:cover-lower}
For the uniform mixture wealth~\eqref{eq:mixV}, every $n\ge1$ with $V_n>0$, and every path,
\begin{equation}\label{eq:coverlower}
\tfrac12\ln V_n-\ln\left(\sqrt{2\pi}\max\{m_0,1-m_0\}\right)\ \le\ R_n\ \le\ \ln(1+n)+1 .
\end{equation}
\end{lemma}
While the uniform mixture achieves $O(\ln n)$ regret on all paths, \citet{orabona2023tight} use a Robbins' prior to construct confidence sequences with $O(\sqrt{V_n\ln \ln V_n}/n)$ widths. In Appendix~\ref{app:OJ}, we show that this mixture achieves an improved regret of $O(\ln\ln  V_n)$ on typical paths, but can incur linear regret on the complement (Remark~\ref{rem:linearOJregret}). This motivates the modified mixture construction introduced in the next section, which achieves $O(\ln\ln V_n)$ regret on typical paths while retaining a worst-case regret of $O(\ln n)$.

\section{Iterated Logarithmic Conditional Regret: Modified Robbins' Mixture}\label{sec:regretbound}
Recall, $I_{m_0}$, $(S_n, V_n)$ from around~\eqref{eq:mixtureprocess} and~\eqref{eq:mixV}, respectively. Let \( W_n := \int\nolimits_{I_{m_0}} W_n(\lambda) \pi(\lambda) d\lambda \), with
\[ \pi(\lambda) = \begin{cases}
    \frac{\ln\ln(6.6e)}{2\lambda \ln\lrp{\frac{6.6e}{(1-m_0) \lambda}}\lrs{ \ln\ln\lrp{ \frac{6.6e}{(1-m_0)\lambda} } }^2}, & 0 < \lambda \le \frac{1}{1-m_0},\\
    \frac{\ln\ln(6.6e)}{2|\lambda| \ln\lrp{\frac{6.6e}{m_0 |\lambda|}} \lrs{ \ln\ln \lrp{\frac{6.6e}{m_0 |\lambda|}} }^2 }, & -\frac{1}{m_0} \le \lambda < 0.
\end{cases} \numberthis\label{eq:mixwealthbdd} \]
The mixture distribution $\pi$ above is a modified Robbins' heavy-near-zero prior \citep[Example 3]{robbins1970statistical}, scaled to have support in $I_{m_0}$. Further, note that it is generally asymmetric around $0$, but it is monotonically decreasing for $\lambda \in (0, \tfrac1{1-m_0}]$ and increasing for $\lambda \in [-\tfrac1{m_0}, 0)$. 

Let $\lambda^*_n$ denote the hindsight-optimal bet that maximizes the wealth, i.e., $$\lambda^*_n \in \argmax_{\lambda\in I_{m_0}}~ \prod_{i=1}^n \lrp{1 - \lambda (X_i - m_0)}\quad \text{ and } \quad W^*_n = W_n(\lambda^*_n).$$ In this section, we present an explicit bound on the regret $R_n:=  \ln W^*_n - \ln W_n$ of the strategy whose wealth at time $n$ is the mixture wealth $W_n$ defined above (mixed with respect to the prior in~\eqref{eq:mixwealthbdd}), with respect to that of the hindsight optimal wealth $W^*_n$. To this end, first recall that whenever $V_n = 0$, we automatically have $S_n = 0$ and $R_n = 0$. So we focus on the case when $V_n > 0$. 

To this end, let $\betal := \min\{m_0, 1-m_0\}$, $\betau := \max\{m_0, 1-m_0\}$,  $\operatorname{Bd}(I_{m_0}) = \{-\frac{1}{m_0}, \frac{1}{1-m_0}\}$ denote the boundary of the set $I_{m_0}$, and $I^\circ_{m_0} = (-\frac{1}{m_0}, \frac{1}{1-m_0})$ denote the interior.

\begin{theorem}\label{th:bddregretbound}
    For all $n\ge 1$ such that $V_n > 0$, 
    \[   
        R_n  \le 
        \begin{cases}
        \frac{2}{\betal^2} + 1 +  \ln\lrp{\frac{8}{\ln\ln(6.6e)}} + \ln\ln\lrp{\frac{14e \betau}{\betal}\sqrt{1+V_n}}\\
        ~~ + 2\ln\ln\ln\lrp{\frac{14 e \betau}{\betal}\sqrt{1+V_n}}, & \text{if } |S_n| \!<\! \sqrt{2V_n} ~\&~ \lambda^*_n \!\in\! I^\circ_{m_0} \\
        \frac{1}{\betal^2} - \ln\lrp{\frac{\ln\ln(6.6e)}{4 \ln(6.6e)[\ln\ln(6.6e)]^2}}, &\text{if } |S_n| \!<\! \sqrt{2V_n}~\&~ \lambda^*_n \!\in \!\operatorname{Bd}(I_{m_0})\\
        1  + \ln\lrp{\frac{4}{\ln\ln (6.6 e)}} + \ln\ln\lrp{\frac{14 e}{\betal}\lrp{1+\sqrt{V_n}}} \\
        ~~ + \ln\lrp{\frac{20 |S_n|}{3 \sqrt{\frac43 |S_n| + 2 V_n}} }  + 2\ln\ln\ln\lrp{\frac{14 e}{\betal}\lrp{1+\sqrt{V_n}}}, &\text{if } \sqrt{2V_n} \le |S_n| \le \frac{\betal}{5}V_n\\
        \frac{1}{2} \ln W^*_n + \ln 4 + \ln\ln(6.6 e) +  \ln\ln\ln(6.6 e) \\
        ~~ + 2\ln\lrp{2\betau + \frac{5}{\betal}}, &\text{if } \sqrt{2V_n} \le |S_n|~\&~ \frac{\betal}{5}V_n < |S_n|.
        \end{cases}\numberthis\label{eq:path-wiseregretrobbins}
    \]
    Moreover, on $\calE_\alpha$, there exists a constant $C(m_0)$ such that 
    \[ \forall n\ge 1, \quad R_n \le \ln\frac{1}{\alpha} + C(m_0) + \frac43 \ln\ln\lrp{\tfrac{14e}{\betal}(1+\sqrt{V_n})} + \frac83 \ln\ln\ln\lrp{\tfrac{14e}{\betal}(1+\sqrt{V_n})}. \numberthis\label{eq:bounded_lnlnt_alpha}  \]
    Furthermore,  if at each time, the data are drawn from a distribution $P$ so that  $(W_n)_{n\ge 1}$ is a non-negative supermartingale, then  $P[\calE_\alpha] \ge 1-\alpha$. 
\end{theorem}
The condition above that $(W_n)_{n\ge 1}$ be an NSM is easily met. For example, $P$ could be supported on $[0,1]$ and satisfy $\Exp{P}{X_n \mid \mathcal F_{n-1}}=m_0$ for all $n \geq 1$. But it could include cases where the conditional mean may not always equal $m_0$, see Appendix~\ref{app:supermartingale}.
\begin{proof}
    We prove the bounds on $R_n$ in the different cases in Lemmas~\ref{lem:bound_smalldrift_interior},~\ref{lem:bound_smalldrift_bdry},~\ref{lem:bounded_meddrift}, and~\ref{lem:bounded_largedrift}. Clearly, the bounds in the first two cases satisfy the inequality in~\eqref{eq:bounded_lnlnt_alpha} unconditionally since these bounds are independent of $S_n$ and scale as $\ln\ln V_n$. For $n\ge 1$ such that either $\sqrt{2V_n} \le |S_n| \le \betal V_n/5$, or $ \sqrt{2V_n} \le |S_n|$ and $\betal V_n/5\le|S_n|$, Corollary~\ref{cor:bounded_meddrift_loglog_Kalpha}, and Lemma~\ref{lem:bounded_largedrift}, respectively, prove the bound in~\eqref{eq:bounded_lnlnt_alpha}. Finally, when the data are (stochastic) such that $\{W_n\}_{n\ge 1}$ is a non-negative supermartingale, Ville's inequality implies that $P[\calE_\alpha] \ge 1-\alpha$, proving the theorem. 
\end{proof}
Theorem~\ref{th:bddregretbound} above proves a path-wise regret bound (Eq.~\eqref{eq:path-wiseregretrobbins}), as well as a conditional regret bound of $O(\ln \ln V_n)$ (Eq.~\eqref{eq:bounded_lnlnt_alpha}) that holds for every $n$ on a set $\calE_\alpha$. If the data are stochastic, then $\calE_\alpha$ is large. We note that the constants such as $6.6e$, $14 e$, $\tfrac{20}3$, etc., are artefacts of our proof and are not optimized. While the worst-case regret in Theorem~\ref{th:bddregretbound} may seem to be linear (last branch), the remark below shows that it is $O(\ln n)$. Thus, the bound in the last branch in~\eqref{eq:path-wiseregretrobbins} can, in particular, be tightened. 

\begin{remark}\label{rem:lognRobbinsregret}
    The worst-case regret of $W_n$ defined using~\eqref{eq:mixwealthbdd} is $O(\ln n)$. Since the prior in~\eqref{eq:mixwealthbdd} is radially decreasing, for all $\lambda\in I_{m_0}$, $\pi(\lambda) \ge c:= \min\{\pi(-1/m_0), \pi(1/(1-m_0))\}$. Clearly, $c\in(0,1)$, and $W_n \ge c \int_{I_{m_0}} W_n(\lambda)  d\lambda$, where the lower bound is a scaled uniform-mixture wealth from Section~\ref{sec:logregret}. This, together with Theorem~\ref{th:ville_gaussianmixture}, implies $R_n = \ln W^*_n - \ln W_n \le \ln(1+n) + 1 - \ln c + \ln(m_0(1-m_0))$. Combining this bound with Theorem~\ref{th:bddregretbound}, gives that $R_n$ is bounded by minimum of the bound in Theorem~\ref{th:bddregretbound} and $\ln(1+n) + 1 - \ln(c m_0(1-m_0))$. 
\end{remark}

In the following theorem, we show that the $O(\ln\ln V_n)$ conditional regret bound holds eventually on the set $\widetilde{\calE_0}$ where the process $W_n$ remains finite, but possibly not uniformly bounded, and $V_n \to \infty$ as $n\to \infty$. This further implies that 
on $\widetilde{\calE_0}$, $|S_n|/V_n \to 0$ and $\limsup_n |S_n|/\sqrt{V_n\ln\ln V_n} \le \sqrt{2}$. In other words, either both of these aforementioned convergences hold or the process $W_n$ satisfies $\limsup_n W_n = \infty$. In particular, if the data are stochastic with appropriate assumptions, the set $\widetilde{\calE_0}$ is a set of measure $1$, and we conclude that the process $W_n$ acts as a witness for the self-normalised SLLN and (upper) LIL. These are formalized in the following theorem.

\begin{theorem}\label{th:lilasregretV}
Consider the event
$\widetilde{\calE_0} := \{ V_n \uparrow \infty;~ \limsup_n W_n < \infty\}.$
We have that on $\widetilde{\calE_0}$, 
\(R_n \leq \ln\ln V_n (1+ o(1))\) eventually. Dividing this by $V_n$ and taking the limit as $n \to \infty$, we get that $\lim_n |S_n|/V_n = 0$ path-wise on $\widetilde{\calE_0}$. Next, dividing by $\ln\ln V_n$ and taking limit as $n \to \infty$, we get  $\limsup_n {|S_n|}/{\sqrt{2 V_n \ln \ln V_n}} \le 1$ path-wise on $\widetilde{\calE_0}$. As a corollary, if at each time, the data are drawn from a distribution $P$  such that $(W_n)_{n\ge 1}$ is a non-negative supermartingale, and if $V_n\to\infty$ almost surely, then $\widetilde{\calE_0}$ is a measure 1 event, thus implying the self-normalised SLLN and upper LIL for bounded data. 
\end{theorem}

As mentioned earlier, the condition above that $(W_n)_{n\ge 1}$ be an NSM is easily met (see Appendix~\ref{app:supermartingale}). As in Theorem~\ref{th:bddregretbound}, Theorem~\ref{th:lilasregretV} demonstrates a deterministic statement of the form: either $R_n = O(\ln\ln V_n)$ (as $V_n \uparrow \infty$), or the wealth process $W_n$ gets arbitrarily large. In case of stochastic data, it provides an explicit supermartingale witnessing the violation of the SLLN and LIL: if either of these is violated on some path with $V_n \uparrow \infty$, then $\limsup_n W_n = \infty$ on that path. See~\Cref{app:proof_lilasregret} for a proof.

\begin{remark}
    Recall that $\calE_\alpha$ denotes the set of paths on which the wealth process is uniformly bounded, while $\calE_0$ allows for eventual boundedness without a uniform bound. Finally, $\widetilde{\calE_0}$ further restricts to the paths with diverging internal variance $V_n$. 
\end{remark}

Finally, in the remark below, we discuss that the rate $\ln\ln V_n$ and the leading constant $1$ in our regret bound in Theorem~\ref{th:lilasregretV} are sharp.

\begin{remark}\label{rem:regret-sharpness}
The iterated-logarithmic regret bound need not be attained on every path. For example, if $X_i=m_0$ for all $i$, then $R_n=0$ for all $n$.
To see a non-trivial example, let 
\[
X_{2k-1}=m_0+a_k,\qquad\text{and}\qquad X_{2k}=m_0-a_k,\]
with $0<a_k\le\min\{m_0,1-m_0\}.$ For every admissible $\lambda$, $W_{2k}(\lambda)
= \prod\nolimits_{j=1}^k (1-\lambda a_j)(1+\lambda a_j) =\prod\nolimits_{j=1}^{k}(1-\lambda^2a_j^2),$ which is at most $1$, and 
$W_{2k-1}(\lambda)\le2$. So the mixture wealth is uniformly bounded by $2$. 
Taking $a_k=b/\sqrt{k}$, where
$0<b\le\min\{m_0,1-m_0\}$, gives
$V_n= \sum_i a^2_i \sim2b^2\ln n$. Our regret bounds above then yield $R_n=O(\ln\ln\ln n)$, which is strictly better than $O(\ln \ln n)$. 

To see the sharpness of the bound in Theorem~\ref{th:lilasregretV}, consider a path in $\tilde{\mathcal E}_0$ satisfying the sharp self-normalised LIL equality
\[
\limsup_{n\to\infty}
\frac{|S_n|}{\sqrt{2V_n\ln\ln V_n}}=1.
\]
Along a subsequence $(n_j)$ attaining this limsup,
$|S_{n_j}|/V_{n_j}\to 0$. Using this and Lemma~\ref{lem:lbwealthstar}, 
\[
\ln W_{n_j}^*
\ge
\frac{S_{n_j}^2}
{(4/3)|S_{n_j}|+2V_{n_j}}
=(1-o(1))\ln\ln V_{n_j}.
\]
Since $\ln W_n$ is bounded above on this path, $\limsup_n R_n/\ln\ln V_n \ge 1$. This lower bound together with Theorem~\ref{th:lilasregretV} yields $\limsup_{n}
{R_n}/{\ln\ln V_n}=1.$ In particular, this equality holds almost surely when the
observations are iid in $[0,1]$ with mean $m_0$ and positive
variance.
\end{remark}

\section{Growth Rates for Mixture Wealth Processes } \label{sec:growthrate}
When the data are drawn iid from a fixed distribution $Q$, the (asymptotic) growth rate of process $W_n$ against $Q$ (defined through the lower limit $\liminf_n \tfrac{1}{n}\ln W_n$)  quantifies the rate at which the wealth grows exponentially under $Q$. 
In this section, we compare the asymptotic growth rates of the uniform mixture in~\eqref{eq:mixV}, the modified Robbins mixture in~\eqref{eq:mixwealthbdd}, and the restricted Robbins mixture of \citet{orabona2023tight}, presented in Appendix~\ref{app:OJRegretFull},
 at the level of their growth rates. 

Unlike the rest of this paper, in this section we assume that the data are drawn iid from a fixed distribution $Q$ with support in $[0,1]$ and mean $m$, for some $m\in [0,1]$ and $m\ne m_0$. Notice that for such stochastic $Q$, $\lim_n V_n/n = \operatorname{Var}(Q) + (m-m_0)^2$ almost surely and $\lim_n |S_n|/n = |m-m_0|$, almost surely (by SLLN), where $\operatorname{Var}(Q)$ is the variance of distribution $Q$.

\paragraph{Uniform mixture.}
Write $W_n^{\mathrm U}$ for the uniform-mixture wealth from~\eqref{eq:mixV} and $R_n^{\mathrm U}:=\ln W_n^*-\ln W_n^{\mathrm U}$. Since  $W_n^{\mathrm U}\le W_n^*$,
Theorem~\ref{th:ville_gaussianmixture} gives $0\le R_n^{\mathrm U}\le\ln(1+n)+1.$ The empirical dual identity and continuity of $\KLinf$ therefore yield
\begin{equation}
\frac{\ln W_n^{\mathrm U}}{n}
=\KLinf(\hat Q_n,m_0)-\frac{R_n^{\mathrm U}}{n}
\longrightarrow\KLinf(Q,m_0)
\qquad\text{Q-almost surely},
\label{eq:loggrowthrate}
\end{equation}
as $n\to \infty$. Thus, the uniform mixture attains the oracle's asymptotic growth rate.

\paragraph{Modified Robbins' mixture.} Let $W_n^{\mathrm R}$ denote the mixture using the prior in~\eqref{eq:mixwealthbdd}. Remark~\ref{rem:lognRobbinsregret} establishes $R_n^{\mathrm R}=O(\ln n)$ on every bounded sequence. Applying the same identity as above gives
\[
\lim_n \frac{\ln W_n^{\mathrm R}}{n}= \KLinf(Q,m_0)
\qquad\text{Q-almost surely}. \numberthis\label{eq:modrobbinsgrowthrate}
\]
The modified Robbins' mixture therefore preserves the uniform mixture's asymptotic growth rate while providing the sharper typical-path regret guarantees of Theorems~\ref{th:bddregretbound} and~\ref{th:lilasregretV}.

\paragraph{\citet{orabona2023tight}'s Mixture. } In Proposition~\ref{prop:lilregretoje_alpha} and Remark~\ref{rem:linearOJregret} in Appendix~\ref{app:OJRegretFull}, we show that this wealth process improves upon the regret of uniform mixture on the set $\calE_\alpha$ (it is $O(\ln\ln V_n)$ as opposed to $O(\ln n)$ for uniform mixture). However, it can suffer a linear regret and a strictly smaller growth rate. In fact, Remark~\ref{rem:linearOJregret} presents an explicit example to demonstrate this. 

\section{Resolving Apparent Trade-offs Between Regret and Growth}
\label{sec:tradeoffs}

Section~\ref{sec:growthrate} shows that the modified Robbins mixture matches the uniform mixture's asymptotic growth while retaining logarithmic regret on every bounded sequence.
Its improved typical-path regret therefore requires neither a loss of worst-case protection nor a loss of growth under iid alternatives.
We now clarify the apparent trade-offs suggested by restricted Robbins mixtures and explain how aggregation provides another way to combine these guarantees.

\subsection{The apparent trade-offs}\label{sec:apparenttradeoffs}

Comparing the uniform mixture with the restricted Robbins mixture of
\citet{orabona2023tight} suggests two apparent trade-offs. The second
also arises in the sub-Gaussian setting of \citet{agrawal2025regret}.

\begin{enumerate}
\item \textbf{Worst-case versus typical-path regret.}
For bounded data, the uniform mixture provides $O(\ln n)$ regret on
every sequence. However, Lemma~\ref{lem:cover-lower} rules out
$O(\ln\ln V_n)$ regret for this mixture on paths with $V_n\to\infty$.
The restricted Robbins mixture (Appendix~\ref{app:OJRegretFull}) achieves $O(\ln\ln V_n)$ regret on
its Ville events, but can incur linear regret against the full
comparator class on some other paths
(Proposition~\ref{prop:lilregretoje_alpha} and
Remark~\ref{rem:linearOJregret}).

In the sub-Gaussian setting of \citet{agrawal2025regret}, the Gaussian
and Robbins mixtures provide logarithmic and iterated-logarithmic
typical-path regret, respectively, measured in the cumulative
variance proxy. Neither provides a uniform $O(\ln n)$ regret
guarantee over arbitrary unbounded sequences. Thus, the comparison
between a uniform protection and a sharper typical-path
regret has no direct analogue for these two mixtures.

\item \textbf{Typical-path regret versus alternative growth.}
For bounded data, the restricted Robbins mixture's improved
typical-path regret coexists with a smaller asymptotic growth
rate under iid alternatives whose growth-optimal fixed bet lies
outside its support. The resulting growth gap corresponds to linear regret against the full comparator class on certain paths.
Section~\ref{sec:growthrate} discusses this gap, with an explicit
example in Remark~\ref{rem:linearOJregret}.

A related regret-growth tension appears in the sub-Gaussian setting.
Under the alternative assumptions considered by
\citet{agrawal2025regret}, the Gaussian mixture attains the oracle's
asymptotic growth rate, whereas the restricted Robbins mixture can
lose growth when the limiting optimal bet lies outside its support.
\end{enumerate}

The above comparisons involve two distinct stochastic regimes: the Ville events have high probability under the conditional-mean null, whereas our growth comparison concerns iid alternatives. For bounded data, our modified Robbins mixture resolves both the apparent trade-offs by retaining substantial prior mass near zero while dominating a positive multiple of the uniform prior. It therefore attains a uniform logarithmic worst-case regret, sharper typical-path regret, and optimal asymptotic growth under alternatives. 

The next subsection presents a general aggregation principle that combines multiple component strategies, preserving the sharper path-wise regret guarantees and the better alternative growth rate. While this aggregation principle works in general, for the sub-Gaussian setting, in particular, it gives an explicit way to resolve the regret-growth trade-off. Table~\ref{tab:comparison} in Appendix~\ref{app:table} summarises these comparisons.

\subsection{Best-of-both-worlds mixture}\label{sec:BOB}
For bounded data, the modified Robbins mixture already combines worst-case protection, improved typical-path regret, and optimal growth.
A fixed allocation of wealth between the uniform and restricted Robbins mixtures provides an alternative construction.
Such combinations are also briefly discussed by \citet{orabona2023tight}; the following proposition makes their regret and growth guarantees explicit. See~\Cref{app:proof_prop:BOB} for the proof.

\begin{proposition}\label{prop:BOB} Let $W_n^{(1)}$ and $W_n^{(2)}$ denote the uniform and restricted Robbins wealth processes from~\eqref{eq:mixV} and~\eqref{eq:mixwealthbdd_OJ}, respectively, each initialized at one.  Fix $s_0\in(0,1)$ and define
$W_n:=s_0W_n^{(1)}+(1-s_0)W_n^{(2)}$. For $\alpha\in(0,1)$, let $\mathcal E_\alpha :=\{\sup_{n\ge1}W_n\le1/ \alpha\}.$ Then the following hold.
\begin{enumerate}
\item Define
$R_n:=\ln W_n^*-\ln W_n$
and
$R_n^{(i)}:=\ln W_n^*-\ln W_n^{(i)}$. On every path, for all $n\ge1$,
\begin{equation}
R_n\le
\min\{ R_n^{(1)}-\ln{s_0}, R_n^{(2)}-\ln({1-s_0})\}.
\label{eq:combregretbound}
\end{equation}
Moreover, on $\mathcal E_\alpha$, there exist finite constants
$K_{\alpha,s_0}$, $C_1>0$, and $C_2>0$ such that
\begin{equation}
\forall n\ge 1, \qquad R_n\le K_{\alpha,s_0}
+C_1\ln\ln({14e}(1+\sqrt{V_n}))
+C_2\ln\ln\ln({14e}(1+\sqrt{V_n})),
\label{eq:lil_agg}
\end{equation}
where $V_n:=\sum_{i=1}^n(X_i-m_0)^2$.
\item Suppose the data are iid from a distribution $Q$ supported on $[0,1]$, with mean $m\ne m_0$.
Let $G:=\liminf_{n}({\ln W_n})/{n}$ and $G^{(i)}:=\liminf_{n}({\ln W_n^{(i)}})/{n}$. Then
\begin{equation}
G\ge\max\{G^{(1)},G^{(2)}\}
\qquad\text{almost surely}.
\label{eq:combgrowthrate}
\end{equation}
\end{enumerate}
Thus, the aggregate achieves the best-of-both-worlds performance, preserving the logarithmic worst-case regret, iterated-logarithmic regret on Ville events, and an optimal asymptotic growth rate.
\end{proposition}

\section{Conclusions}\label{sec:conclusions}

We studied regret and asymptotic growth for mixture wealth processes
in sequential betting. For bounded data, our explicit full-support
Robbins mixture combines $O(\ln n)$ regret on every sequence,
$O(\ln\ln V_n)$ regret on typical paths with $V_n\to\infty$, and
optimal asymptotic wealth growth under iid alternatives.
The regret guarantees are deterministic and the stochastic assumptions are only used to 
establish that the relevant Ville events have high probability
or probability one. The same wealth process also provides an explicit game-theoretic witness for the self-normalised strong law of large
numbers and upper law of iterated logarithm.

More generally, splitting initial wealth between two strategies
preserves the smaller component regret against a common comparator,
up to an additive constant, and achieves an asymptotic growth rate
at least as large as either component's. This aggregation principle
resolves the apparent tension between typical-path regret and
alternative growth in both the bounded and sub-Gaussian settings.
These results connect online learning, anytime-valid inference,
and game-theoretic probability through the analysis of betting wealth.

\paragraph{Limitations and future directions.}
Our explicit full-support Robbins construction and local-curvature
bounds rely on bounded observations, and several finite-time
constants are not optimized. The aggregation principle applies
more generally, including to the unbounded mixtures of
\citet{agrawal2025regret}. Extending our direct mixture analysis to
other betting families, such as sub-exponential and sub-Poisson
processes, is a direction for future work. 

\section{AI Disclosure}
All the core results were obtained without using AI tools. GPT-6.1-Sol and Opus-5.5 helped with handling a corner case in Theorem~\ref{th:bddregretbound}, sharpened constants in~\eqref{eq:bounded_lnlnt_alpha}, suggested Lemma~\ref{lem:cover-lower}, and created Table~\ref{tab:comparison}, in addition to helping with improving the exposition and checking proofs. The authors retain responsibility for the correctness and relevance of the content.

\section*{Acknowledgements}
{SA acknowledges the generous support from the Pratiksha Trust, Bangalore, through the Young Investigator Award, ANRF's grant ANRF/ECRG/2025/000560/ENS, and the DST INSPIRE Faculty Grant IFA24-ENG-389. AR acknowledges funding from NSF grant 2310718 and a Sloan fellowship.}

\bibliography{BibTex}

\begin{thebibliography}{19}
\providecommand{\natexlab}[1]{#1}
\providecommand{\url}[1]{\texttt{#1}}
\expandafter\ifx\csname urlstyle\endcsname\relax
  \providecommand{\doi}[1]{doi: #1}\else
  \providecommand{\doi}{doi: \begingroup \urlstyle{rm}\Url}\fi

\bibitem[Agrawal(2023)]{thesis}
Shubhada Agrawal.
\newblock \emph{Bandits with Heavy Tails: Algorithms Analysis and Optimality}.
\newblock PhD thesis, Tata Institute of Fundamental Research, 2023.
\newblock URL \url{http://hdl.handle.net/10603/478863}.

\bibitem[Agrawal and Ramdas(2026)]{agrawal2025regret}
Shubhada Agrawal and Aaditya Ramdas.
\newblock Eventually {LIL} regret: Almost sure $\ln\ln {T}$ regret for a
  sub-gaussian mixture on unbounded data.
\newblock In \emph{Proceedings of The 37th International Conference on
  Algorithmic Learning Theory}, volume 313, pages 1--26. PMLR, 23--26 Feb 2026.

\bibitem[Agrawal et~al.(2020)Agrawal, Juneja, and Glynn]{agrawal2020optimal}
Shubhada Agrawal, Sandeep Juneja, and Peter Glynn.
\newblock Optimal $\delta$-correct best-arm selection for heavy-tailed
  distributions.
\newblock In \emph{Algorithmic Learning Theory}, pages 61--110. PMLR, 2020.

\bibitem[Agrawal et~al.(2021)Agrawal, Koolen, and Juneja]{agrawal2021optimal}
Shubhada Agrawal, Wouter~M Koolen, and Sandeep Juneja.
\newblock Optimal best-arm identification methods for tail-risk measures.
\newblock \emph{Advances in Neural Information Processing Systems},
  34:\penalty0 25578--25590, 2021.

\bibitem[Cover(1991)]{cover1991universal}
Thomas~M Cover.
\newblock Universal portfolios.
\newblock \emph{Mathematical Finance}, 1\penalty0 (1):\penalty0 1--29, 1991.

\bibitem[Cover and Ordentlich(1996)]{cover2002universal}
Thomas~M Cover and Erik Ordentlich.
\newblock Universal portfolios with side information.
\newblock \emph{IEEE Transactions on Information Theory}, 42\penalty0
  (2):\penalty0 348--363, 1996.

\bibitem[Darling and Robbins(1967)]{darling1967confidence}
Donald~A Darling and Herbert Robbins.
\newblock Confidence sequences for mean, variance, and median.
\newblock \emph{Proceedings of the National Academy of Sciences}, 58\penalty0
  (1):\penalty0 66--68, 1967.

\bibitem[Honda and Takemura(2010)]{honda2010asymptotically}
Junya Honda and Akimichi Takemura.
\newblock An asymptotically optimal bandit algorithm for bounded support
  models.
\newblock In \emph{Conference on Learning Theory}, pages 67--79, 2010.

\bibitem[Howard et~al.(2021)Howard, Ramdas, Mcauliffe, and
  Sekhon]{howard2021time}
Steven~R Howard, Aaditya Ramdas, Jon Mcauliffe, and Jasjeet Sekhon.
\newblock Time-uniform, nonparametric, nonasymptotic confidence sequences.
\newblock \emph{The Annals of Statistics}, 49\penalty0 (2):\penalty0
  1055--1080, 2021.

\bibitem[Kaufmann and Koolen(2021)]{kaufmann_mm_21}
Emilie Kaufmann and Wouter~M. Koolen.
\newblock Mixture martingales revisited with applications to sequential tests
  and confidence intervals.
\newblock \emph{Journal of Machine Learning Research}, 22\penalty0
  (246):\penalty0 1--44, 2021.

\bibitem[Koolen and Van~Erven(2015)]{koolen2015second}
Wouter~M Koolen and Tim Van~Erven.
\newblock Second-order quantile methods for experts and combinatorial games.
\newblock In \emph{Conference on Learning Theory}, pages 1155--1175. PMLR,
  2015.

\bibitem[Larsson et~al.(2026)Larsson, Ramdas, and Ruf]{larsson2025variables}
Martin Larsson, Aaditya Ramdas, and Johannes Ruf.
\newblock Testing hypotheses generated by constraints.
\newblock \emph{Mathematics of Operations Research}, 2026.

\bibitem[Orabona and Jun(2023)]{orabona2023tight}
Francesco Orabona and Kwang-Sung Jun.
\newblock Tight concentrations and confidence sequences from the regret of
  universal portfolio.
\newblock \emph{IEEE Transactions on Information Theory}, 70, 2023.

\bibitem[Orabona and P{\'a}l(2016)]{orabona2016coin}
Francesco Orabona and D{\'a}vid P{\'a}l.
\newblock Coin betting and parameter-free online learning.
\newblock \emph{Advances in Neural Information Processing Systems}, 29, 2016.

\bibitem[Ramdas and Wang(2025)]{ramdas2024hypothesis}
Aaditya Ramdas and Ruodu Wang.
\newblock Hypothesis testing with e-values.
\newblock \emph{Foundations and Trends{\textregistered} in Statistics},
  1\penalty0 (1-2):\penalty0 1--390, 2025.

\bibitem[Robbins(1970)]{robbins1970statistical}
Herbert Robbins.
\newblock Statistical methods related to the law of the iterated logarithm.
\newblock \emph{The Annals of Mathematical Statistics}, 41\penalty0
  (5):\penalty0 1397--1409, 1970.

\bibitem[Shafer and Vovk(2005)]{shafer2005probability}
Glenn Shafer and Vladimir Vovk.
\newblock \emph{Probability and finance: it's only a game!}, volume 491.
\newblock John Wiley \& Sons, 2005.

\bibitem[Ville(1939)]{ville1939etude}
Jean Ville.
\newblock \emph{Etude critique de la notion de collectif}.
\newblock Gauthier-Villars Paris, 1939.

\bibitem[Waudby-Smith and Ramdas(2024)]{waudby2024estimating}
Ian Waudby-Smith and Aaditya Ramdas.
\newblock Estimating means of bounded random variables by betting.
\newblock \emph{Journal of the Royal Statistical Society Series B: Statistical
  Methodology}, 86\penalty0 (1):\penalty0 1--27, 2024.

\end{thebibliography}
\appendix
\section{Proofs from Section~\ref{sec:logregret}}\label{app:proof_warmuplogTregret}
\begin{proof}[Proof of~\Cref{th:ville_gaussianmixture}]
   The unconditional path-wise regret bound in~\eqref{eq:regretbound} follows immediately from \citet[Lemma F.1]{agrawal2021optimal} (by setting $\Gamma = I_{m_0}$, $d=1$, and $g_i(\lambda) = \ln(1-\lambda(X_i -m_0))$). Next, using the lower bound on $\ln W^*_n$ from Lemma~\ref{lem:lbwealthstar}, we have
   \begin{align*}
       \frac{S^2_n}{\frac{4}{3}|S_n| + 2V_n} - \ln W_n \le R_n \le \ln(1+n) + 1.
   \end{align*}
    Dividing by $n$ and taking limits, we obtain
    \[
    \limsup_{n\to\infty}
    \frac{\frac{|S_n|}{V_n}}{\frac{4}{3}\frac{|S_n|}{V_n}+2}\,
    \frac{|S_n|}{n}
    =0,
    \qquad \text{path-wise on }\mathcal E_0.
    \]
    Using the above, we will now argue via contradiction that $|S_n|/n \to 0$ path-wise on $\mathcal E_0$. Let 
    \[ 
    a_n:=\frac{|S_n|}{V_n},\qquad b_n:=\frac{|S_n|}{n},
    \qquad
    f(x):=\frac{x}{\tfrac43 x+2}.
    \]
    Then $f$ is continuous, nonnegative, increasing on $[0,\infty)$, and
    $f(x)=0$ iff $x=0$. Suppose $b_n \nrightarrow 0$ on some path in $\mathcal E_0$. Then there exist $\delta>0$ and a
    subsequence $(n_k)_{k\ge 1}$ such that, on that path, $b_{n_k}\ge \delta$ for all $k$. Since
    \[
    \limsup_n f(a_n)b_n = 0,
    \]
    it follows that $f(a_{n_k})b_{n_k}\to 0$, and hence, using $b_{n_k}\ge \delta$,
    we get $f(a_{n_k})\to 0$. Therefore $a_{n_k}\to 0$. But then
    \[
    \frac{V_{n_k}}{n_k}
    =
    \frac{|S_{n_k}|/n_k}{|S_{n_k}|/V_{n_k}}
    =
    \frac{b_{n_k}}{a_{n_k}}
    \ge \frac{\delta}{a_{n_k}}
    \to \infty,
    \]
    which contradicts the fact that $V_n/n\in[0,1]$ for all $n$. Hence
    $|S_n|/n\to 0$ path-wise on $\mathcal E_0$.
\end{proof}

\begin{proof}[Proof for Lemma~\ref{lem:cover-lower}]
    The upper bound in~\eqref{eq:coverlower} is Theorem~\ref{th:ville_gaussianmixture}. We give a short proof for the lower bound. Let $f_n(\lambda):=\sum_{i=1}^n\ln(1-\lambda(X_i-m_0))$ and $\lambda_n^*\in\argmax_{\lambda\in I_{m_0}}f_n(\lambda)$, so that $\ln W_n^*=f_n(\lambda_n^*)$. Clearly, $0\le 1-\lambda(x-m_0)\le 1/\beta_l$, with equalities for $(\lambda, x)\in \{(-1/m_0,0), (1/(1-m_0), 1)\}$. Hence, for $\lambda \in I_{m_0}$ such that $f_n(\lambda) > - \infty$, 
\[
f_n''(\lambda)=-\sum_{i=1}^n\frac{(X_i-m_0)^2}{\bigl(1-\lambda(X_i-m_0)\bigr)^2}\ \le\ -\beta_l^2V_n.
\]
Since $\lambda_n^*$ maximizes $f_n$ over the convex set $I_{m_0}$, it follows that for $\lambda\in I_{m_0}$,  $f_n(\lambda)\le\ln W_n^*-\tfrac12\beta_l^2V_n(\lambda-\lambda_n^*)^2$. Note that this bound trivially holds whenever $f_n(\lambda) = -\infty$. Next, the uniform density on $I_{m_0}$ equals $m_0(1-m_0)=\beta_l\beta_u$, so
\[
W_n=\beta_l\beta_u\int_{I_{m_0}}e^{f_n(\lambda)}\,d\lambda
\ \le\ \beta_l\beta_u\,W_n^*\int_{\mathbb R}e^{-\beta_l^2V_nu^2/2}\,du
\ =\ \frac{\sqrt{2\pi}\,\beta_u}{\sqrt{V_n}}\,W_n^*.
\]
Taking logarithms gives the lower bound in~\eqref{eq:coverlower}. 
\end{proof}

\section{Proofs from Section~\ref{sec:regretbound}}

In this appendix, we will frequently use the following notation:
\[ \lambda^*_n = \argmax\limits_{\lambda\in \lrs{-\frac{1}{m_0}, \frac{1}{1-m_0}}}~ \sum\limits_{i=1}^{n} \ln(1-\lambda(X_i - m_0)), \]
    \[ \alpha_n = \begin{cases}
        1-m_0, &\text{ if } \lambda^*_n \ge 0\\
        m_0, & \text{ if } \lambda^*_n < 0
    \end{cases} \quad \text{and} \quad \beta_n = \begin{cases}
        m_0, &\text{ if } \lambda^*_n \ge 0\\
        1-m_0, & \text{ if } \lambda^*_n < 0.
    \end{cases} \]
Further, $\betau = \max\{m_0, 1-m_0\}$, $\betal = \min\{m_0, 1-m_0\}$, and throughout this appendix, we assume $V_n > 0$. 

\begin{lemma}\label{lem:regbound1} 
    When $\lambda^*_n \ne 0$, the regret of the mixture strategy in~\eqref{eq:mixwealthbdd} is bounded as follows:
    \[R_n \le \frac12 \ln W^*_n - \ln\lrp{\pi(\lambda^*_n) |\lambda^*_n|}. \]
\end{lemma}
\begin{proof}
    Consider the following inequalities for any $\frac{1}{1-m_0} \ge \lambda_2 > \lambda_1 \ge 0$. From radial monotonicity of $\pi$ on the positive domain, we have 
    \begin{align*}
        W_n 
        &\ge \int\limits_{\lambda_1}^{\lambda_2} W_n(\lambda) \pi(\lambda) d\lambda \tag{$W_n(\lambda) \ge 0$} \\
        &\ge \pi(\lambda_2) \int\limits_{\lambda_1}^{\lambda_2} W_n(\lambda) d\lambda\\
        &= (\lambda_2 - \lambda_1) \pi(\lambda_2) \int\limits_{0}^1 W_n( \lambda_1 (1-a) + \lambda_2 a  ) da\\
        &\ge (\lambda_2 - \lambda_1) \pi(\lambda_2) \int\limits_{0}^1 \lrp{W_n(\lambda_1)}^{1-a} \lrp{W_n(\lambda_2)}^{a} da \tag{log-concavity of $W_n$}\\
        &= \pi(\lambda_2) (\lambda_2 - \lambda_1) \frac{W_n(\lambda_2) - W_n(\lambda_1)}{\ln \frac{W_n(\lambda_2)}{W_n(\lambda_1)}}. \numberthis\label{eq:ineq1pos}
    \end{align*}
    Similarly, for $-\frac{1}{m_0} \le \lambda_2 < \lambda_1 < 0$, we also have
    \[ W_n \ge \int\limits_{\lambda_2}^{\lambda_1} W_n(\lambda) \pi(\lambda) d\lambda \ge \pi(\lambda_2) (\lambda_1 - \lambda_2)\frac{W_n(\lambda_2) - W_n(\lambda_1)}{\ln \frac{W_n(\lambda_2)}{W_n(\lambda_1)}} \numberthis\label{eq:ineq1neg}  \]
    Using the inequality~\eqref{eq:ineq1pos} with $\lambda_2 = \lambda^*_n$ and $\lambda_1 = 0$ when $\lambda^*_n > 0$ and otherwise, using~\eqref{eq:ineq1neg} with $\lambda_2 = \lambda^*_n$ and $\lambda_1 = 0$, we get
    \begin{align*}
        W_n &\ge  \pi(\lambda^*_n)|\lambda^*_n| \frac{W^*_n - W_n(0)}{\ln \frac{W^*_n}{W_n(0)}} \tag{From~\eqref{eq:ineq1pos} and~\eqref{eq:ineq1neg}}\\
        &= \pi(\lambda^*_n)|\lambda^*_n| \frac{W^*_n - 1}{\ln{W^*_n}} \tag{$W_n(0)=1$}\\
        &\ge \pi(\lambda^*_n)|\lambda^*_n| \sqrt{W^*_n}.  \tag{Since $\frac{x-1}{\ln{x}} \ge \sqrt{x}$}
    \end{align*}
    Taking the log on both sides and rearranging, we get the desired inequality. 
\end{proof}

\begin{lemma}\label{lem:regbound2.5}
    For $-\frac{1}{m_0} < \lambda^*_n < \frac{1}{1-m_0}$, $\lambda^*_n \ne 0$, and any $0 < \rho_n$, the regret is bounded as below:
    \[ R_n = \ln W^*_n  - \ln W_n \le  \frac{\rho^2_n V_n}{2 \lrp{1-\alpha_n|\lambda^*_n|}^2}  - \ln(\min\lrset{\rho_n, |\lambda^*_n|} \pi(\lambda^*_n)).\]
\end{lemma}

\begin{proof}
    For $\lambda \in [-\frac{1}{m_0},\frac{1}{1-m_0}]$, let $f_n(\lambda) := \ln W_n(\lambda)$. For $\lambda^*_n \in (-\frac{1}{m_0}, \frac{1}{1-m_0})$, consider 
    \[ I_n := \begin{cases}
        [\lambda^*_n  - \rho_n, \lambda^*_n], &\text{ when } \rho_n < \lambda^*_n\\
        [0, \lambda^*_n],  & \text{ when } 0 < \lambda^*_n \le \rho_n\\
        [\lambda^*_n, 0], & \text{ when } - \rho_n \le \lambda^*_n < 0\\
        [\lambda^*_n, \lambda^*_n + \rho_n], &\text{ when } \lambda^*_n < -\rho_n.
    \end{cases} \]
    Clearly, $I_n \subset (-\frac{1}{m_0},\frac{1}{1-m_0})$, and $f'_n(\lambda^*_n) = 0$ since $\lambda^*_n$ lies in the interior. Moreover, for any $\lambda \in I_n$, $\lambda$ and $\lambda^*_n$ have the same sign, and for some $\lambda'\in I_n$, we have
    \begin{align*}
        f_n(\lambda)
        &= f_n(\lambda^*_n) + \frac{(\lambda-\lambda^*_n)^2}{2} f''_n(\lambda') \tag{$f'_n(\lambda^*_n) = 0$} \\
        &=f_n(\lambda^*_n) - \frac{(\lambda-\lambda^*_n)^2}{2}  \sum\limits_{i=1}^n \frac{(X_i - m_0)^2}{(1-\lambda'(X_i-m_0))^2} \\
        &\ge f_n(\lambda^*_n) - \frac{(\lambda-\lambda^*_n)^2}{2}   \sum\limits_{i=1}^n \frac{(X_i - m_0)^2}{(1- \alpha_n|\lambda^*_n|)^2}\\
        &= f_n(\lambda^*_n) - \frac{(\lambda-\lambda^*_n)^2}{2}   \frac{V_n}{(1-\alpha_n|\lambda^*_n|)^2}\\
        &\ge f_n(\lambda^*_n) - \frac{\rho^2_n V_n}{2(1-\alpha_n|\lambda^*_n|)^2}, \tag{$|\lambda-\lambda^*_n| \le \rho_n$}
    \end{align*}
    where the first inequality follows since $|\lambda'| \le |\lambda^*_n|$ for $\lambda' \in I_n$, and $X_i - m_0 \in [-m_0, 1-m_0]$. Now, we will use the above inequality to lower-bound the wealth of the mixture strategy. 
    \begin{align*}
        W_n = \int\limits_{\lrs{-\frac{1}{m_0}, \frac{1}{1-m_0}}} W_n(\lambda) \pi(\lambda) d\lambda 
        &\ge \int\limits_{I_n} e^{f_n(\lambda)} \pi(\lambda)d\lambda\\
        &\ge e^{f_n(\lambda^*_n)} e^{\frac{-\rho^2_n V_n}{2(1-\alpha_n|\lambda^*_n|)^2}} \int\limits_{I_n} \pi(\lambda) d\lambda\\
        &\ge W^*_n e^{\frac{-\rho^2_n V_n}{2(1-\alpha_n|\lambda^*_n|)^2}} \pi(\lambda^*_n) \min\lrset{\rho_n, |\lambda^*_n|}.
    \end{align*}
    Taking the logarithm and rearranging it, we get the desired inequality. 
\end{proof}

\begin{remark}
    Lemma~\ref{lem:regbound2.5} gives a regret bound that is linear in $V_n$. We will later choose $\rho_n \propto 1/\sqrt{V_n}$ to obtain the desired bounds in various cases.
\end{remark}

\begin{lemma}\label{lem:boundlambdastar}
    If $\lambda^*_n \in (-\frac{1}{m_0}, \frac{1}{1-m_0})$, then $\lambda^*_n S_n \le 0$. Further, 
    \[ \frac{|S_n|}{V_n}\lrp{1-\alpha_n |\lambda^*_n|}^2 \le |\lambda^*_n| \le \frac{|S_n|}{V_n}\lrp{1+\beta_n |\lambda^*_n|}^2.\]
    
    Moreover, the above further implies
    \[  |\lambda^*_n| \ge  \frac{|S_n|}{V_n + 2 \alpha_n |S_n|} =\begin{cases}
        \frac{|S_n|}{ V_n + 2(1-m_0)|S_n|}, &\text{ if } \lambda^*_n > 0\\
        \frac{|S_n|}{V_n + 2m_0 |S_n|}, & \text{ if } \lambda^*_n < 0.
    \end{cases}
    \]
\end{lemma}
\begin{proof} First, note that the lemma holds trivially for $\lambda^*_n = 0$ since in this case, $S_n = 0$.

For $\lambda \in [-\frac{1}{m_0},\frac{1}{1-m_0}]$, let $f_n(\lambda) := \sum\limits_{i=1}^n \ln(1-\lambda(X_i - m_0))$. Since $\lambda^*_n \in (-\frac{1}{m_0}, \frac{1}{1-m_0})$, from Taylor's theorem we have the following for some $\lambda \in [0, \lambda^*_n]$ (to be read as $[0,\lambda^*_n]$, if $\lambda^*_n > 0$ and $[\lambda^*_n, 0]$, otherwise): 
    \begin{align*}
        0 = f'_n(\lambda^*_n) = f'_n(0) + f''_n(\lambda) \lambda^*_n,
    \end{align*}
    which implies 
    \[ S_n := \sum\limits_{i=1}^n (X_i - m_0) = - \lambda^*_n \sum\limits_{i= 1}^n \frac{(X_i - m_0)^2}{(1-\lambda(X_i - m_0))^2}. \numberthis\label{eq:der0} \]
    It is clear from the above that $\lambda^*_n$ and $S_n$ have opposite signs. We now split the cases.\\

    \noindent\textbf{Case 1  ($\lambda^*_n > 0$):} First, observe from~\eqref{eq:der0} that, in this case, $S_n < 0$. Further, since $(X_i - m_0) \in [-m_0, 1-m_0]$, and $\lambda \in [0, \lambda^*_n]$, we have
    \[ 0 < 1-(1-m_0)\lambda^*_n \le (1-\lambda(X_i - m_0)) \le 1 + \lambda^*_n m_0.  \]
    Using this in~\eqref{eq:der0}, we get
    \[ \frac{|S_n|}{V_n}(1-(1-m_0)\lambda^*_n)^2 \le \lambda^*_n \le \frac{|S_n|}{V_n}(1+m_0 \lambda^*_n)^2.\]
    Finally, the lower bound in the above set of inequalities gives (since, $(1-z)^2 \ge 1-2z$)
    \[ \lambda^*_n \ge \frac{|S_n|}{V_n}(1-2(1-m_0)\lambda^*_n), \]
    which implies
    \[ \lambda^*_n \ge \frac{|S_n|}{V_n + 2 (1-m_0)|S_n|},\]
    proving all the required bounds in this case.\\

    \noindent\textbf{Case 2 ($\lambda^*_n < 0$):} In this case, from~\eqref{eq:der0}, $S_n > 0$. Further, since $(X_i - m_0) \in [-m_0, 1-m_0]$, and $\lambda \in [\lambda^*_n, 0]$, we have
    \[ 1-|\lambda^*_n| m_0 \le (1-\lambda(X_i - m_0)) \le 1+|\lambda^*_n|(1-m_0). \]
    Using these in~\eqref{eq:der0}, we get
    \[  \frac{|S_n|}{V_n}(1-m_0|\lambda^*_n|)^2 \le |\lambda^*_n| \le \frac{|S_n|}{V_n} (1+(1-m_0)|\lambda^*_n|)^2. \]
    Again, using the lower bound in the above set of inequalities, we get
    \[ |\lambda^*_n| \ge \frac{|S_n|}{V_n} (1-2 m_0|\lambda^*_n|), \]
    which further implies
    \[ |\lambda^*_n| \ge \frac{|S_n|}{V_n + 2m_0 |S_n|}. \]
    This completes the proof for the lemma.
\end{proof}

\begin{lemma}\label{lem:boundryimpliesratiobound}
    For $\lambda^*_n \in \{-\frac{1}{m_0},  \frac{1}{1-m_0}\}$, we have $\lambda^*_n S_n \le 0$. Furthermore, we have 
    \[ \frac{V_n}{|S_n|} \le |\lambda^*_n| \le \max\lrset{ \frac{1}{m_0}, \frac{1}{1-m_0} }. \]
\end{lemma}
\begin{proof}
    Clearly, the right most inequality holds for $\lambda^*_n \in \{-\tfrac{1}{m_0}, \tfrac{1}{1-m_0}\}$. It thus suffices to prove only the left most inequality. 

    For $\lambda \in [-\frac{1}{m_0}, \frac{1}{1-m_0}]$, let $f_n(\lambda) := \sum_{i=1}^n \ln(1-\lambda(X_i - m_0))$. Then, 
    \[ f'_n(\lambda) = - \sum\limits_{i=1}^n \frac{X_i - m_0}{1-\lambda(X_i - m_0)} \quad \text{and}\quad f''_n(\lambda) = -\sum\limits_{i=1}^n \frac{(X_i - m_0)^2}{(1-\lambda(X_i - m_0))^2}. \]
    When $\lambda^*_n$ is on the boundary, since it maximizes $f_n$ over $I_{m_0}$, the one-sided first-order optimality conditions give $f'_n(\lambda^*_n) \ge 0$ when $\lambda^*_n = \frac{1}{1-m_0}$, and $f'_n(\lambda^*_n) \le 0$ otherwise. 

    Using Taylor's theorem, we have for some $\lambda \in [0,\lambda^*_n]$ (i.e., $[0,\lambda^*_n]$ when $\lambda^*_n > 0$ and $[\lambda^*_n, 0]$ when $\lambda^*_n < 0$): 
    \[ 0 \le f'_n(\lambda^*_n) = f'_n(0) + \frac{f''_n(\lambda)}{1-m_0}, \quad \text{ when } \lambda^*_n = \frac{1}{1-m_0}, \]
    and
    \[ 0 \ge f'_n(\lambda^*_n) = f'_n(0) - \frac{f''_n(\lambda)}{m_0}, \quad \text{ when } \lambda^*_n = -\frac{1}{m_0}.  \]
    Since $f'_n(0) = -S_n$, the above can be rewritten as
    \[  S_n \le \frac{f''_n(\lambda)}{1-m_0}, \quad \text{ when } \lambda^*_n = \frac{1}{1-m_0}, \]
    and
    \[ S_n  \ge - \frac{f''_n(\lambda)}{m_0}, \quad \text{ when } \lambda^*_n = -\frac{1}{m_0}.  \]
    Since $f''_n(\cdot) < 0$, the above two inequalities imply that $\lambda^*_n$ and $S_n$ have the opposite sign, which further implies that for some $\lambda\in [0,\lambda^*_n]$, 
    \[ |S_n| \ge \begin{cases}
        - \frac{f''_n(\lambda)}{1-m_0}, & \quad \text{ when } \lambda^*_n = \frac{1}{1-m_0}\\
        - \frac{f''_n(\lambda)}{m_0}, & \quad \text{ when } \lambda^*_n = -\frac{1}{m_0}. 
    \end{cases} \]
    Further, for $\lambda \in [0,\lambda^*_n]$ and for all $i\in [n]$, 
    \[ 1-\lambda(X_i - m_0) \le \begin{cases}
        1+ \frac{m_0}{1-m_0}, & \text{ when }\lambda^*_n = \frac{1}{1-m_0}\\
        1+\frac{1-m_0}{m_0}, & \text{ when }\lambda^*_n = -\frac{1}{m_0}
    \end{cases} = \begin{cases}
        \frac{1}{1-m_0}, & \text{ when } \lambda^*_n = \frac{1}{1-m_0}\\
        \frac{1}{m_0}, & \text{ when } \lambda^*_n = -\frac{1}{m_0}.
    \end{cases} \]
    Using this in the bound on $|S_n|$, we get
    \[ |S_n| \ge \begin{cases}
        \frac{V_n (1-m_0)^2}{1-m_0}, &\text{ when }\lambda^*_n = \frac{1}{1-m_0}\\
        \frac{V_n m^2_0}{m_0}, &\text{ when }\lambda^*_n = -\frac{1}{m_0}
    \end{cases} = \begin{cases}
        V_n (1-m_0), &\text{ when }\lambda^*_n = \frac{1}{1-m_0}\\
        V_n m_0, &\text{ when }\lambda^*_n = -\frac{1}{m_0},
    \end{cases} \]
    proving the lemma. 
\end{proof}

In the following, we use the bounds in Lemmas~\ref{lem:regbound1} and~\ref{lem:regbound2.5} above to get more explicit bounds on $R_n$ in terms of $V_n$ and $n$. 

\begin{lemma}\label{lem:bound_smalldrift_interior}
    For $|S_n| < \sqrt{2V_n}$ and $\lambda^*_n \in (-\frac{1}{m_0}, \frac{1}{1-m_0})$, 
    \begin{align*}
        R_n &\le \frac{2}{\alpha^2_n} + 1 +  \ln\lrp{\frac{8}{\ln\ln(6.6e)}} + \ln\ln\lrp{\frac{14e \betau}{\alpha_n}\sqrt{1+V_n}} + 2\ln\ln\ln\lrp{\frac{14 e \betau}{\alpha_n}\sqrt{1+V_n}}\\
        &\le \frac{2}{\betal^2} + 1 +  \ln\lrp{\frac{8}{\ln\ln(6.6e)}} + \ln\ln\lrp{\frac{14e \betau}{\betal}\sqrt{1+V_n}} + 2\ln\ln\ln\lrp{\frac{14 e \betau}{\betal}\sqrt{1+V_n}}.
    \end{align*}
\end{lemma}
\begin{proof}
First, suppose that $\lambda_n^*=0$. Since $0$ lies in the interior of $I_{m_0}$, first-order optimality gives $S_n=0$. Thus $W_n^*=1$ and $R_n=-\ln W_n$. Set
\[
\rho_n=\frac{1}{2\beta_u\sqrt{1+V_n}},
\qquad
C_n=\frac{14e\beta_u}{1-m_0}\sqrt{1+V_n}.
\]
For every $\lambda\in[0,\rho_n]$, we have
$|\lambda(X_i-m_0)|\le 1/2$. Using
$\ln(1-x)\ge -x-x^2$ for $|x|\le 1/2$, we obtain
\[
\ln W_n(\lambda)
\ge -\lambda S_n-\lambda^2V_n
\ge -\rho_n^2V_n
\ge -1.
\]
Consequently, using the prior in~\eqref{eq:mixwealthbdd},
\begin{align*}
W_n
&\ge e^{-1}\int_0^{\rho_n}\pi(\lambda)\,d\lambda\\
&= \frac{e^{-1}\ln\ln(6.6e)}
{2\ln\ln\!\left(\frac{6.6e}{(1-m_0)\rho_n}\right)}\\
&\ge \frac{e^{-1}\ln\ln(6.6e)}{2\ln\ln C_n}.
\end{align*}
Therefore,
\[
R_n\le
1+\ln\!\left(\frac{2}{\ln\ln(6.6e)}\right)
+\ln\ln\ln C_n.
\]
Since $C_n>e^e$, this is no larger than the claimed bound with $\alpha_n=1-m_0$. Hence, the bound in the lemma holds in this case.

For the remainder of the proof, we assume $\lambda_n^*\ne0$. For $\lambda\in [-\frac{1}{m_0},\frac{1}{1-m_0}]$, let $f_n(\lambda) := \ln W_n(\lambda) = \sum_{i=1}^n \ln(1-\lambda(X_i - m_0))$. Clearly, 
    \[f'_n(0) = -S_n \qquad \text{and} \qquad f''_n(\lambda) = - \sum\limits_{i=1}^n \frac{(X_i - m_0)^2}{(1-\lambda(X_i - m_0))^2}.\]
    Since $R_n := \ln W^*_n - \ln W_n$, to get a bound on $R_n$, we will show that, in this case, $\ln W^*_n$ is bounded by a constant, and the mixture log-wealth $\ln W_n$ is not too small. 
    
    \noindent\textbf{ Step 1 (Upper bound on $\ln W^*_n$).} Using $\ln(1-x) \le -x$ for $x < 1$, 
    \begin{align*}
        \ln W^*_n = \sum\limits_{i=1}^n \ln(1-\lambda^*_n (X_i - m_0))\le - \lambda^*_n S_n &= |\lambda^*_n| |S_n| \\
        &\le \frac{|S_n|}{V_n} \lrp{1+ {\beta_n}|\lambda^*_n|}^2 |S_n| \tag{Lemma~\ref{lem:boundlambdastar}}\\
        &\le \frac{S^2_n}{\alpha^2_n V_n} \tag{$1 + {\beta_n}|\lambda^*_n| \le 1/\alpha_n$}\\
        &\le \frac{2}{\alpha^2_n} \tag{$|S_n| < \sqrt{2V_n}$},
    \end{align*}
    where, recall that $\beta_n = m_0$ when $\lambda^*_n > 0$, and $1-m_0$ when $\lambda^*_n < 0$. 

    \noindent \textbf{Step 2 (Lower bound on $\ln W_n$).} The idea is to choose a small interval close to $0$ to lower bound the integration over the entire range of the bets $\lambda$ with an integration on this shorter interval. Let 
    $$\rho_n = \frac{1}{2\betau\sqrt{1+V_n}},$$
    and define the interval $J_n$ as (recall, $\lambda^*_n$ and $S_n$ have opposite sign): 
    \[ J_n = \begin{cases}
        [\tfrac{\rho_n}{2}, \rho_n], & \text{ if } S_n \le 0\\
        [-\rho_n, -\tfrac{\rho_n}{2}], & \text{ if } S_n > 0 
    \end{cases} = \begin{cases}
        [\tfrac{\rho_n}{2}, \rho_n], & \text{ if } \lambda^*_n \ge 0\\
        [-\rho_n, -\tfrac{\rho_n}{2}], & \text{ if } \lambda^*_n < 0
        \end{cases}. \]
    Then, for $\lambda \in J_n$, first observe that $-\lambda S_n \ge 0$. Further, we have $|\lambda(X_i-m_0)| \le \rho_n\betau \le \frac{1}{2}$. Further, since $\ln(1-x) \ge - x - x^2$ for $|x| \le \frac{1}{2}$, we have 
    $$\ln(1-\lambda(X_i-m_0)) \ge -\lambda (X_i - m_0) - \lambda^2(X_i - m_0)^2,$$ 
    which further implies that for all $\lambda \in J_n$, 
    \begin{align*} \ln W_n(\lambda) = \sum\limits_{i=1}^n \ln(1-\lambda(X_i - m_0)) &\ge -\lambda S_n - \lambda^2 V_n \\
    &\ge -\lambda^2 V_n \tag{$-\lambda S_n \ge 0$}\\
    &\ge -\rho^2_n V_n \tag{$|\lambda| \le \rho_n$}\\
    &\ge -\frac{1}{4 \betau^2}\\
    &\ge -1 \tag{Since $\betau \ge \tfrac{1}{2}$}.
    \end{align*} 
    Now, consider the following inequalities, which follow since $\pi$ is radially decreasing and $|J_n| = \frac{\rho_n}{2}$: 
    \begin{align*}
        W_n 
        \ge \int\limits_{J_n} W_n(\lambda) \pi(\lambda) d\lambda \ge e^{-1}\int\limits_{J_n}\pi(\lambda) d\lambda&\ge \begin{cases}
            e^{-1}\frac{\rho_n}{2} \pi(\rho_n), & \text{ if } S_n \le 0\\
            e^{-1}\frac{\rho_n}{2}\pi(-\rho_n), & \text{ if } S_n > 0
        \end{cases}\\
        &\ge \frac{e^{-1}\ln\ln(6.6e)}{8 \ln\lrp{\frac{6.6e}{\alpha_n\rho_n}}\lrs{\ln\ln\lrp{\frac{6.6e}{\alpha_n\rho_n}}}^2}.
    \end{align*}
    Finally, since $\frac{1}{\alpha_n \rho_n}=\frac{2\betau}{\alpha_n} \sqrt{1+V_n}$, plugging this in the above lower bound, we get
    \[ \ln W_n \ge -1 + \ln\lrp{\frac{\ln\ln(6.6e)}{8}} - \ln\ln\lrp{\frac{14e \betau}{\alpha_n}\sqrt{1+V_n}} - 2\ln\ln\ln\lrp{\frac{14 e \betau}{\alpha_n}\sqrt{1+V_n}}. \]
    
    \noindent\textbf{Step 3 (Regret bound).} Finally, combining the bounds from Steps 1 and 2,  
    \begin{align*}
        R_n 
        &= \ln W^*_n - \ln W_n \\
        &\le \frac{2}{\alpha^2_n} + 1 +  \ln\lrp{\frac{8}{\ln\ln(6.6e)}} + \ln\ln\lrp{\frac{14e \betau}{\alpha_n}\sqrt{1+V_n}} + 2\ln\ln\ln\lrp{\frac{14 e \betau}{\alpha_n}\sqrt{1+V_n}}\\
        &\le \frac{2}{\betal^2} + 1 +  \ln\lrp{\frac{8}{\ln\ln(6.6e)}} + \ln\ln\lrp{\frac{14e \betau}{\betal}\sqrt{1+V_n}} + 2\ln\ln\ln\lrp{\frac{14 e \betau}{\betal}\sqrt{1+V_n}},
    \end{align*}
    proving the bound. 
\end{proof}

\begin{lemma}\label{lem:bound_smalldrift_bdry}
    For $|S_n| < \sqrt{2V_n}$ and $\lambda^*_n \in \{-\frac{1}{m_0}, \frac{1}{1-m_0}\}$, 
    \[ R_n \le - \ln\lrp{\frac{\ln\ln(6.6e)}{4 \ln(6.6e)[\ln\ln(6.6e)]^2}} + \frac{1}{\alpha^2_n} \le - \ln\lrp{\frac{\ln\ln(6.6e)}{4 \ln(6.6e)[\ln\ln(6.6e)]^2}} + \frac{1}{\betal^2}.\] 
\end{lemma}
\begin{proof}
    From Lemma~\ref{lem:regbound1}, when $\lambda^*_n$ is on the boundary, we have $R_n \le 0.5 \ln W^*_n - \ln(\pi(\lambda^*_n) |\lambda^*_n|)$. We now upper bound $\ln W^*_n$ below. Using $\ln(1-x) \le -x$,
    \begin{align*}
        \ln W^*_n = \sum\limits_{i=1}^n \ln\lrp{1 - \lambda^*_n(X_i-m_0)} \le -\lambda^*_n S_n \le |\lambda^*_n||S_n|.
    \end{align*}
    Further, from Lemma~\ref{lem:boundryimpliesratiobound} and the condition in the lemma statement, 
    \[ \frac{V_n}{|\lambda^*_n|} \le |S_n| \le \sqrt{2 V_n} \implies V_n \le 2(\lambda^*_n)^2. \]
    Combining this with the bound on $\ln W^*_n$, we have
    \[  \frac{1}{2}\ln W^*_n \le (\lambda^*_n)^2 \le \frac{1}{\alpha^2_n} = \begin{cases}
        \frac{1}{(1-m_0)^2}, & \text{ if } \lambda^*_n = \frac{1}{1-m_0}\\
        \frac{1}{m^2_0}, & \text{ if } \lambda^*_n = -\frac{1}{m_0}.
    \end{cases} \numberthis\label{eq:lbwstarn} \]
    We now consider the other term in the bound on $R_n$, 
    \begin{align*}
        \ln(|\lambda^*_n| \pi(\lambda^*_n)) = \ln\lrp{\frac{\ln\ln(6.6e)}{2 \ln(6.6e)[\ln\ln(6.6e)]^2}}.
    \end{align*}
    Using both the above bounds in the bound on $R_n$, we get the desired inequality. 
\end{proof}

\begin{lemma}\label{lem:bounded_meddrift}
    For $\sqrt{2V_n} \le |S_n| \le \tfrac{\betal}{5} V_n$, $\lambda^*_n$ lies in the interior, i.e., $\lambda^*_n \in (-\tfrac{1}{m_0},\tfrac{1}{1-m_0})$. Moreover, 
    \[ |\lambda^*_n| \le \frac{|S_n|}{V_n} + 5\beta_n \lrp{\frac{|S_n|}{V_n}}^2,\numberthis\label{eq:ineqdeltat}\]
    and
    \begin{align*} 
    R_n &\le 1 + \ln\lrp{ \frac{20 |S_n|}{3 \sqrt{\tfrac{4}{3}|S_n| + 2 V_n}} } + \ln\lrp{\frac{4}{\ln\ln (6.6 e)}} \\
    &\qquad + \ln\ln\lrp{\frac{14 e}{\alpha_n}\lrp{1+\sqrt{V_n}}} + 2\ln\ln\ln\lrp{\frac{14 e}{\alpha_n}\lrp{1+\sqrt{V_n}}}. 
    \end{align*}
\end{lemma}
\begin{proof}
    First, suppose $\lambda^*_n \in \{-\frac{1}{m_0}, \frac{1}{1-m_0}\}$. Then, from Lemma~\ref{lem:boundryimpliesratiobound}, we have that
     \[ |S_n| \ge V_n/|\lambda^*_n| \ge \alpha_n V_n \ge \betal V_n,\]
    which contradicts the assumption on $|S_n|$ in the lemma statement. Thus, $\lambda^*_n \in (-\tfrac{1}{m_0}, \tfrac{1}{1-m_0})$.

    \noindent\textbf{Regret bound. } We will now prove the regret bound, assuming that~\eqref{eq:ineqdeltat}. Let $\rho_n$ be as given below:
    \[ \rho_n = \min\lrset{|\lambda^*_n|, \frac{1-\alpha_n |\lambda^*_n|}{\sqrt{1+V_n}}} \implies \rho_n = \begin{cases}
        |\lambda^*_n|, & \text{ when } |\lambda^*_n| < \frac{1}{\alpha_n+\sqrt{1+V_n}}\\
        \frac{1- \alpha_n|\lambda^*_n|}{\sqrt{1+V_n}}, & \text{ otherwise}. 
    \end{cases} \]
    Since $\rho_n \le |\lambda^*_n|$, Lemma~\ref{lem:regbound2.5} gives
    \begin{align*} 
        R_n 
        &\le \frac{1}{2} - \ln\lrp{\rho_n \pi(\lambda^*_n)} \\
        &\le \frac{1}{2}+ \ln\frac{|\lambda^*_n|}{\rho_n} + \ln\lrp{\frac{4}{\ln\ln (6.6 e)}} + \ln\ln\lrp{\frac{6.6 e}{\alpha_n|\lambda^*_n|}} + 2\ln\ln\ln\lrp{\frac{6.6 e}{\alpha_n|\lambda^*_n|}}.\numberthis\label{eq:regretboundint3}
    \end{align*}
    Note that since $\lambda^*_n$ is in the interior and $|S_n| \ge \sqrt{2 V_n}$, we have from Lemma~\ref{lem:boundlambdastar} that 
    \[ \alpha_n |\lambda^*_n| \ge \frac{\alpha_n}{2 \alpha_n + \frac{V_n}{|S_n|}}\ge \frac{\alpha_n}{2\alpha_n +\sqrt{V_n/2}} \ge \frac{\alpha_n}{2\lrp{1 + \sqrt{V_n}}}.\]
    Using the above bound in~\eqref{eq:regretboundint3}, we further get
    \begin{align*} 
        R_n 
        &\le \frac{1}{2}+ \ln\frac{|\lambda^*_n|}{\rho_n} + \ln\lrp{\frac{4}{\ln\ln (6.6 e)}} \\
        &\qquad\qquad + \ln\ln\lrp{{\frac{14 e}{\alpha_n}}\lrp{1+\sqrt{V_n}}} + 2\ln\ln\ln\lrp{{\frac{14 e}{\alpha_n}}\lrp{1+\sqrt{V_n}}}.\numberthis\label{eq:regretboundint4}
    \end{align*}
    We now analyze the two cases in the definition of $\rho_n$ separately. 
    
    \noindent\textbf{Case 1: $|\lambda^*_n| < \frac{1}{\alpha_n+\sqrt{1+V_n}}$. } Using the definition of $\rho_n$, we get
    \[ R_n \le \frac{1}{2} + \ln\lrp{\frac{4}{\ln\ln (6.6 e)}} + \ln\ln\lrp{{\frac{14 e}{\alpha_n}}\lrp{1+\sqrt{V_n}}} + 2\ln\ln\ln\lrp{\frac{14 e}{\alpha_n}\lrp{1+\sqrt{V_n}}}. \]

    \noindent\textbf{Case 2: $|\lambda^*_n| \ge \frac{1}{\alpha_n+\sqrt{1+V_n}}$. } In this case, using the definition of $\rho_n$, we get
    \begin{align*} 
        R_n 
        &\le \frac{1}{2}+ \ln\frac{|\lambda^*_n|\sqrt{1+V_n}}{1-\alpha_n|\lambda^*_n|} + \ln\lrp{\frac{4}{\ln\ln (6.6 e)}} \\
        &\qquad\qquad + \ln\ln\lrp{\frac{14 e}{\alpha_n}\lrp{1+\sqrt{V_n}}} + 2\ln\ln\ln\lrp{\frac{14 e}{\alpha_n}\lrp{1+\sqrt{V_n}}}. \numberthis\label{eq:regboundcase1}
    \end{align*}
    Next, since $\frac{|S_n|}{V_n} \le \frac{\betal}{5} \le \frac{1}{5}$, we get the following from~\eqref{eq:ineqdeltat}: 
    \[ |\lambda^*_n| \le \frac{2|S_n|}{V_n} \le \frac{2}{5} \implies \frac{|\lambda^*_n|}{1-\alpha_n|\lambda^*_n|} \le \frac{2|S_n|}{V_n(1-\frac{2}{5})}  = \frac{10|S_n|}{3 V_n}. \]
    Using this in~\eqref{eq:regboundcase1}, we get that in this case,
    \begin{align*}
        R_n 
        &\le 1 + \ln\lrp{\frac{10|S_n|}{3\sqrt{\frac{4}{3}|S_n| + 2V_n}}.\frac{\sqrt{\frac{4}{3}|S_n| + 2V_n}}{\sqrt{V_n}}} + \ln\lrp{\frac{4}{\ln\ln (6.6 e)}} \\
        &\qquad\qquad \qquad + \ln\ln\lrp{\frac{14 e}{\alpha_n}\lrp{1+\sqrt{V_n}}} + 2\ln\ln\ln\lrp{\frac{14 e}{\alpha_n}\lrp{1+\sqrt{V_n}}}\\
        &\le 1 + \ln\lrp{\frac{20|S_n|}{3\sqrt{\frac{4}{3}|S_n| + 2V_n}}} + \ln\lrp{\frac{4}{\ln\ln (6.6 e)}} \\
        &\qquad\qquad\qquad + \ln\ln\lrp{\frac{14 e}{\alpha_n}\lrp{1+\sqrt{V_n}}} + 2\ln\ln\ln\lrp{\frac{14 e}{\alpha_n}\lrp{1+\sqrt{V_n}}}, \tag{$|S_n| \le V_n/5$}
    \end{align*}
    proving the regret bound in the lemma, assuming~\eqref{eq:ineqdeltat}. \\
    
    \noindent \textbf{Proving~\eqref{eq:ineqdeltat}.} It essentially  follows from solving the upper bound on $|\lambda^*_n|$ in Lemma~\ref{lem:boundlambdastar}, and using the fact that $ \alpha_n V_n/|S_n| \ge  \betal V_n/|S_n| \ge 5$. A similar inequality was proven in  \citet[Pg. 448]{orabona2023tight}, but we present a proof here, for completeness. 

    From Lemma~\ref{lem:boundlambdastar}, 
    \[|\lambda^*_n| \le \frac{|S_n|}{V_n}(1+\beta_n |\lambda^*_n|)^2.\]
    Solving for the quadratic in $|\lambda^*_n|$, we either get a lower bound on or an upper bound on $|\lambda^*_n|$. However, since $\alpha_n V_n/|S_n| \ge 5$, we get that the lower bound obtained above, is greater than $1/(1-m_0)$, hence, invalid. Thus, $|\lambda^*_n|$ is upper bounded by the smaller root of the resulting quadratic giving
    \begin{align*} 
    \beta_n |\lambda^*_n| 
    &\le \frac{V_n}{2 \beta_n |S_n|} - 1 - \frac{1}{2} \sqrt{ \frac{V^2_n}{\beta^2_n|S^2_n|} - \frac{4 V_n}{\beta_n|S_n|} }\\
    &= \frac{ \beta_n |S_n|   }{V_n } \!+\! \lrp{\frac{\beta_n |S_n|}{V_n}}^2 \!\lrp{ \frac{V^3_n}{2\beta^3_n|S^3_n|}\! -\! \frac{V^2_n}{\beta^2_n|S^2_n|} \!-\! \frac{V^2_n}{2\beta^2_n|S^2_n|}\sqrt{ \frac{V^2_n}{\beta^2_n|S^2_n|} \!-\! \frac{4 V_n}{\beta_n |S_n|}  } \!-\! \frac{V_n}{\beta_n |S_n|} }\\
    &\le \frac{ \beta_n |S_n|   }{V_n } + 5\lrp{\frac{\beta_n |S_n|}{V_n}}^2,
    \end{align*}
    where the last inequality follows since the term multiplying $\beta^2_n|S^2_n|/V^2_n$ on the rhs is a non-increasing function of $V_n/(\beta_n|S_n|)$, and its value at $V_n /(\beta_n|S_n|) = 5$ is bounded from above by $5$. This completes the proof.
\end{proof} 
\begin{corollary}\label{cor:bounded_meddrift_loglog_Kalpha}
On the event $\mathcal E_\alpha$, there exists a constant
$K_\alpha$ such that for $n \ge 1$ and $ \sqrt{2V_n} \le |S_n| \le \betal V_n/5$, $R_n \le  K_\alpha +  \frac43\ln\ln(\tfrac{14e}{\betal} (1+\sqrt{V_n})) + \frac83 \ln \ln \ln(\tfrac{14e}{\betal}(1+\sqrt{V_n}))$.
\end{corollary}
\begin{proof}
    From Lemma~\ref{lem:lbwealthstar},
    \begin{align*}
        \ln W^*_n \ge \frac{S^2_n}{\frac{4}{3} |S_n| + 2V_n}. 
    \end{align*}
    Using the above bound, on $\calE_\alpha,$ we have
    \begin{align*} \frac{S^2_n}{\frac{4}{3}|S_n| + 2V_n} \le \ln W^*_n 
    &\le \ln\frac{1}{\alpha} + 1 + \ln\frac{20}{3} + \ln\lrp{ \frac{|S_n|}{\sqrt{\frac{4}{3}|S_n| + 2V_n}}} + \ln\lrp{\frac{4}{\ln\ln(6.6e)}} \\
    &\qquad\qquad+ \ln\ln\lrp{\frac{14 e}{\alpha_n}\lrp{1+\sqrt{V_n}}} + 2\ln\ln\ln\lrp{\frac{14 e}{\alpha_n}\lrp{1+\sqrt{V_n}}}\\
    &\le \ln\frac{1}{\alpha} + 1 + \ln\frac{20}{3} + \frac{1}{4} \frac{S^2_n}{{\frac{4}{3}|S_n| + 2V_n}} + \ln\lrp{\frac{4}{\ln\ln(6.6e)}} \\
    &\qquad\qquad + \ln\ln\lrp{\frac{14 e}{\alpha_n}\lrp{1+\sqrt{V_n}}} + 2\ln\ln\ln\lrp{\frac{14 e}{\alpha_n}\lrp{1+\sqrt{V_n}}},
    \end{align*}
    where the last inequality follows since $\ln(x) \le {x^2}/{4}$. On rearranging the right-most and left-most inequalities above, we get
    \begin{align*}
        \frac{S^2_n}{\frac{4}{3}|S_n| + 2V_n} &\le \frac{4}{3}\ln\frac{1}{\alpha} + \frac{4}{3} + \frac{4}{3}\ln\frac{20}{3} + \frac{4}{3} \ln\lrp{\frac{4}{\ln\ln(6.6e)}} \\
        &\qquad\qquad + \frac{4}{3}\ln\ln\lrp{\frac{14 e}{\alpha_n}\lrp{1+\sqrt{V_n}}} + \frac{8}{3}\ln\ln\ln\lrp{\frac{14 e}{\alpha_n}\lrp{1+\sqrt{V_n}}}.
    \end{align*}
    Using this and that $\ln(x) \le x^2/4$ in the bound in Lemma~\ref{lem:bounded_meddrift}, we get that on $\calE_\alpha$, 
    \begin{align*}
        R_n &\le 1 + \ln\frac{20}{3}+\frac{1}{4}.{\frac{|S^2_n|}{{\frac{4}{3}|S_n| + 2V_n}}} + \ln\lrp{\frac{4}{\ln\ln (6.6 e)}}\\
        &\qquad\qquad\qquad+ \ln\ln\lrp{\frac{14 e}{\alpha_n}\lrp{1+\sqrt{V_n}}} + 2\ln\ln\ln\lrp{\frac{14 e}{\alpha_n}\lrp{1+\sqrt{V_n}}}\\
        &\le \frac{1}{3} \ln\frac{1}{\alpha} + \frac43 + \frac{4}{3}\ln\frac{20}{3} + \frac{4}{3}\ln\lrp{\frac{4}{\ln\ln (6.6 e)}}\\
        &\qquad\qquad\qquad+ \frac{4}{3}\ln\ln\lrp{\frac{14 e}{\alpha_n}\lrp{1+\sqrt{V_n}}} + \frac{8}{3}\ln\ln\ln\lrp{\frac{14 e}{\alpha_n}\lrp{1+\sqrt{V_n}}},
    \end{align*}
    proving the lemma statement. 
\end{proof}

\begin{lemma}\label{lem:lbwealthstar}
    For $n\in\N$, 
    \[ \ln W^*_n \ge \sup\limits_{\lambda\in[-1,1]} \ln W_n(\lambda) \ge \frac{S^2_n}{\frac{4}{3}|S_n| + 2V_n}. \]
\end{lemma}
\begin{proof}
    First, observe that for $|\lambda| < 1$, 
    \[ \ln(1 - \lambda x  ) \ge -\lambda x + x^2 \lrp{\ln(1-|\lambda|) + |\lambda| }. \]
    Using this, for any $|\lambda| < 1$, we have
    \begin{align*} 
    \ln W_n(\lambda) 
    &= \sum\limits_{i=1}^n \ln(1-\lambda(X_i - m_0)) \\
    & \ge -\lambda S_n + V_n \lrp{ \ln(1-|\lambda|) + |\lambda| },
    \end{align*}
    and hence, 
    \[ \sup\limits_{\lambda\in (-1,1)} \ln W_n(\lambda) \ge V_n \sup\limits_{\lambda \in (-1,1)}~ \lrp{  \frac{-\lambda S_n}{V_n} +  \ln(1-|\lambda|) + |\lambda| }. \]
    It is easy to verify that the maximizer on the rhs above is $\tilde{\lambda}_n = - S_n / (V_n + |S_n|)$. On substituting this in the above inequality, we get
    \begin{align*} \sup_{\lambda\in [-1,1]}\ln W_n(\lambda) &\ge -\tilde{\lambda}_n S_n + V_n \ln(1-|\tilde{\lambda}_n|) + V_n|\tilde{\lambda}_n|= |S_n| - V_n \ln\lrp{1 + \frac{|S_n|}{V_n}}\\
    &\ge   \frac{S^2_n}{\frac{4}{3}|S_n| + 2V_n },
    \end{align*}
    where, to get the last inequality, we use $\ln(1+|x|) \le |x| \frac{6 + |x|}{6 + 4|x|}$.
\end{proof}

\begin{lemma}\label{lem:bounded_largedrift}
    For $|S_n| \ge \sqrt{2 V_n}$ and $\betal V_n < 5 |S_n|$, we have the following:
    \begin{align*}
        R_n
        &\le \frac{1}{2} \ln W^*_n + \ln 4 + \ln\ln(6.6 e) +  \ln\ln\ln(6.6 e) + 2\ln\lrp{2\alpha_n + \tfrac{5}{\betal}}.\numberthis\label{eq:rnbounduncond} 
    \end{align*}
    Moreover, on $\mathcal E_\alpha$, 
    \[R_n \le \ln\frac{1}{\alpha} + 2\lrp{\ln 4 + \ln\ln(6.6 e) + \ln\ln\ln(6.6 e) + 2\ln\lrp{2\alpha_n + \tfrac{5}{\betal}}}.\]
\end{lemma}
\begin{proof}
    From Lemma~\ref{lem:regbound1}, we have
    \begin{align*}
        R_n
        &\le \frac{1}{2} \ln W^*_n - \ln\lrp{\pi(\lambda^*_n)|\lambda^*_n|}. \numberthis\label{eq:rnbound}
    \end{align*}
    Either $\lambda^*_n$ is on the boundary or in the interior. We will handle the two cases separately. \\
    \noindent\textbf{Case 1 (Boundary). } In this case, $\lambda^*_n 
    \in \{-\frac{1}{m_0}, \frac{1}{1-m_0}\}$, and the bound in~\eqref{eq:rnbound} becomes
    \begin{align*}
        R_n \le \frac{1}{2}\ln W^*_n + \ln\lrp{\frac{4}{\ln\ln (6.6 e)}} + \ln\ln(6.6 e) + 2 \ln\ln\ln(6.6 e),
    \end{align*}
    proving the first inequality in the lemma in this case.
    
    \noindent\textbf{Case 2 (Interior).} In this case, $\lambda^*_n \in (-\frac{1}{m_0}, \frac{1}{1-m_0})$. From Lemma~\ref{lem:boundlambdastar}, we also have
    \begin{align*}
        |\lambda^*_n| \ge \frac{1}{2\alpha_n + \frac{V_n}{|S_n|}} \ge \frac{1}{2\alpha_n + \frac{5}{\betal}} \tag{$V_n \le \frac{5}{\betal} |S_n|$}.
    \end{align*}
    Using this in~\eqref{eq:rnbound},  
    \begin{align*}
        R_n
        &\le \frac{1}{2} \ln W^*_n - \ln\lrp{\frac{\pi(\lambda^*_n)}{2\alpha_n + \frac{5}{\betal}}} \\
        &\le \frac{1}{2} \ln W^*_n + \ln 4 + \ln\ln(6.6 e) + \ln\ln\ln(6.6 e) + \ln\tfrac{1}{\betal} + \ln\lrp{2\alpha_n + \tfrac{5}{\betal}}\\
        &\le \frac{1}{2} \ln W^*_n + \ln 4 + \ln\ln(6.6 e) + \ln\ln\ln(6.6 e) + 2\ln\lrp{2\alpha_n + \tfrac{5}{\betal}},  
    \end{align*}
    proving the first inequality in the lemma in this case. 

    We now prove the bound on $\mathcal E_\alpha$. To this end, recall that $R_n = \ln W^*_n - \ln W_n$. Thus, from~\eqref{eq:rnbounduncond}, 
    \[ \frac{1}{2} \ln W^*_n \le \ln W_n  + \ln 4 + \ln\ln(6.6 e) +  \ln\ln\ln(6.6 e) + 2\ln\lrp{2\alpha_n + \tfrac{5}{\betal}}, \]
    which on $\calE_\alpha$, gives
    \begin{align*} 
        R_n 
        &\le \ln W_n + 2\lrp{\ln 4 + \ln\ln(6.6 e) + \ln\ln\ln(6.6 e) + 2\ln\lrp{2\alpha_n + \tfrac{5}{\betal}}} \\
        &\le \ln\frac{1}{\alpha}  + 2\lrp{\ln 4 + \ln\ln(6.6 e) + \ln\ln\ln(6.6 e) + 2\ln\lrp{2\alpha_n + \tfrac{5}{\betal}}}. 
        \end{align*}
    This completes the proof.
\end{proof}

\subsection{Proof of~\Cref{th:lilasregretV}}\label{app:proof_lilasregret}
\begin{proof} Since $R_n = \ln W^*_n - \ln W_n$, from~\eqref{eq:path-wiseregretrobbins} we get a lower bound on $\ln W_n$. Further lower bounding $\ln W^*_n$ using Lemma~\ref{lem:lbwealthstar},  using $\ln(x) \le x-1$, and completing the squares, we get the following path-wise lower bound on $\ln W_n$:
\[   
        \ln W_n  \ge  
        \begin{cases}
        {\frac{S^2_n/V_n}{\frac43(|S_n|/V_n) + 2 }}  -  \ln\lrp{\frac{8}{\ln\ln(6.6e)}}\\
        \quad - \ln\ln\lrp{\frac{14e \betau}{\betal}\sqrt{1+V_n}} - \frac{2}{\betal^2} - 1 \\
        \quad - 2\ln\ln\ln\lrp{\frac{14 e \betau}{\betal}\sqrt{1+V_n}}, & \text{if } |S_n| < \sqrt{2V_n} ~\&~ \lambda^*_n \in I^\circ_{m_0} \\
        {\frac{S^2_n/V_n}{\frac43(|S_n|/V_n) + 2}} - \frac{1}{\betal^2}\\
        \quad + \ln\lrp{\frac{\ln\ln(6.6e)}{4 \ln(6.6e)[\ln\ln(6.6e)]^2}}, &\text{if } |S_n| < \sqrt{2V_n}~\&~ \lambda^*_n \in \operatorname{Bd}(I_{m_0})\\
        \lrp{\frac{|S_n|/\sqrt{V_n}}{\sqrt{\frac{4}{3}(|S_n|/V_n) + 2}} - \frac{1}{2} }^2 - \ln\frac{20}{3} \\
        \qquad - \frac14  - \ln\lrp{\frac{4}{\ln\ln (6.6 e)}} \\ 
        \qquad - \ln\!\ln\!\lrp{\frac{14 e}{\betal}\lrp{1+\sqrt{V_n}}}  \\
        \qquad - 2\ln\!\ln\!\ln\!\lrp{\frac{14 e}{\betal}\lrp{1+\sqrt{V_n}}}, &\text{if } \sqrt{2V_n} \le |S_n| \le \frac{\betal}{5}V_n\\
        \frac{V_n}{2}\frac{S^2_n/V^2_n}{\frac{4}{3}({|S_n|}/{V_n}) + 2} -  \ln 4 - \ln\ln(6.6 e)\\
        ~~ -  \ln\ln\ln(6.6 e)  - 2\ln\lrp{2\betau + \tfrac{5}{\betal}}, &\text{if } \sqrt{2V_n} \le |S_n|~\&~ \frac{\betal}{5}V_n < |S_n|.
        \end{cases}\numberthis\label{eq:lbonWn}
    \]
We now argue that on $\widetilde{\calE_0}$, the last branch occurs only finitely many times.
First, observe that on this branch, the first term on the right-hand side is monotonically increasing in $|S_n|/V_n$. Since $|S_n|/V_n > \tfrac{\betal}{5} $, we can further lower bound as
\[ \ln W_n \ge \frac{V_n}{2}. \frac{\betal^2}{25 (\tfrac{4\betal}{15} + 2 )} - \ln 4 - \ln\ln(6.6e)- \ln\ln\ln(6.6e) - 2\ln\lrp{2\betau + \tfrac{5}{\betal}}. \]
Now, recall that on $\widetilde{\calE_0}$, $V_n \uparrow \infty$ and $\limsup_n W_n <\infty$. From this and the above inequality, we see that on $\widetilde{\calE_0}$, the last branch occurs only finitely many times.  Consequently, for all sufficiently large \(n\), one of the first three branches applies. Thus, we only focus on the first three cases, henceforth. 

Now, on $\widetilde{\calE_0}$, dividing both sides of~\eqref{eq:lbonWn} by $V_n$ and taking limit as $n\to\infty$ we get $|S_n|/V_n \to 0$. Similarly, dividing by $\ln\ln V_n$ instead, and taking the limit as $n\to\infty$, we get 
\[\limsup_n \frac{|S_n|}{\sqrt{2 V_n\ln\ln V_n}} \le 1. \numberthis\label{eq:snlil} \]

Finally, from~\eqref{eq:path-wiseregretrobbins}, using $\ln(x) \le x -1 $, we also have the following eventually on $\widetilde{\calE_0}$ (recall, last branch only occurs finitely many times):

    \[   
        R_n  \le 
        \begin{cases}
        \frac{2}{\betal^2} + 1 +  \ln\lrp{\frac{8}{\ln\ln(6.6e)}} + \ln\ln\lrp{\frac{14e \betau}{\betal}\sqrt{1+V_n}}\\\quad + 2\ln\ln\ln\lrp{\frac{14 e \betau}{\betal}\sqrt{1+V_n}}, & \text{if } |S_n| < \sqrt{2V_n} ~\&~ \lambda^*_n \in I^\circ_{m_0} \\
        \frac{1}{\betal^2} - \ln\lrp{\frac{\ln\ln(6.6e)}{4 \ln(6.6e)[\ln\ln(6.6e)]^2}}, &\text{if } |S_n| < \sqrt{2V_n}~\&~ \lambda^*_n \in \operatorname{Bd}(I_{m_0})\\
        \ln\lrp{\frac{4}{\ln\ln (6.6 e)}} + \ln\ln\lrp{\frac{14 e}{\betal}\lrp{1+\sqrt{V_n}}} \\
        \quad + {\frac{20 |S_n|}{3 \sqrt{\frac43 |S_n| + 2 V_n}} }  + 2\ln\ln\ln\lrp{\frac{14 e}{\betal}\lrp{1+\sqrt{V_n}}}, &\text{if } \sqrt{2V_n} \le |S_n| \le \frac{\betal}{5}V_n
        \end{cases}
    \]

Dividing the above inequality by $\ln\ln V_n$, and using~\eqref{eq:snlil} in the  bound, on $\widetilde{\calE_0}$, $R_n \le \ln\ln V_n (1+o(1))$, eventually. 
\end{proof}

\section{Proof of Proposition~\ref{prop:BOB}}\label{app:proof_prop:BOB}
\begin{proof}
Recall, $W_n = s_0 W^{(1)}_n + (1-s_0)W^{(2)}_n$, and hence, 
\[
W_n \ge s_0 W_n^{(1)}
\qquad\text{and}\qquad
W_n \ge (1-s_0) W_n^{(2)}.\numberthis\label{eq:wealthlb}
\]
Thus, we have $\ln W_n \ge \max\{\ln(s_0 W_n^{(1)}),\ln( (1-s_0) W_n^{(2)})\}$, and therefore,
\begin{align*}
R_n := \ln W_n^* - \ln W_n
&\le
\min\lrset{\ln W^*_n - \ln W^{(1)}_n - \ln s_0, \ln W^*_n - \ln W^{(2)}_n-\ln(1-s_0)}\\
&=  \min\lrset{R^{(1)}_n - \ln{{s_0}}, R_n^{(2)}-\ln\lrp{{1-s_0}}},
\end{align*}
implying the bound in~\eqref{eq:combregretbound}. \\

\noindent Next, from~\eqref{eq:wealthlb}, on $\calE_\alpha$, we also have
\[ \ln(1-s_0) +  \sup_{n\ge 1}   \ln W^{(2)}_n \le \ln \tfrac{1}{\alpha}. \]
Thus, for $\calE^{(2)}_{\alpha} := \{ \sup_{n\ge 1} W^{(2)}_n \le \tfrac{1}{\alpha} \} $,  we have $\calE_\alpha \subseteq \calE^{(2)}_{(1-s_0) \alpha}$. Combining this with~\eqref{eq:combregretbound} and Proposition~\ref{prop:lilregretoje_alpha}, on $\calE_\alpha$, there exist constants $K_\alpha$, $C_1 > 0$, and $C_2 > 0$ such that 
\[ R_n \le R^{(2)}_n + \ln\tfrac{1}{1-s_0} \le K_{\alpha} + \ln\tfrac{1}{1-s_0} + C_1 \ln \ln \lrp{14e(1+\sqrt{V_n})} + C_2 \ln\ln\ln\lrp{14e(1+\sqrt{V_n})},\]
proving~\eqref{eq:lil_agg}.\\

\noindent Next, observe from~\eqref{eq:wealthlb} that
\[
\frac{\ln W_n}{n} \ge \frac{\ln s_0}{n} + \frac{\ln W^{(1)}_n}{n} \quad \text{ and }\quad \frac{\ln W_n}{n}  \ge \frac{\ln (1-s_0)}{n} + \frac{\ln W^{(2)}_n}{n}.
\]
Taking \(n\to\infty\), 
\[
G:=\liminf_{n\to\infty}\frac{\ln W_n}{n}
\ge
\max\left\{
\liminf_{n\to\infty}\frac{\ln W_n^{(1)}}{n},
\liminf_{n\to\infty}\frac{\ln W_n^{(2)}}{n}
\right\},
\]
we get~\eqref{eq:combgrowthrate}, completing the proof.
\end{proof}

\section{The Bounded Betting Game}\label{app:bettinggame}

Consider a game between two risk-neutral players: a buyer (named Skeptic) and a seller (named Forecaster) of bets. In this game, Forecaster believes that the conditional distribution of the unseen outcome is some $P\in \mathcal P[0,1]$ with mean $m_P=m_0$. Skeptic is skeptical of  Forecaster's belief about the mean. He instead believes that its conditional distribution is $Q\in\mathcal P[0,1]$ with mean $m_Q \ne m_0$. 

Let $I_{m_0} = [-{1}/{m_0}, {1}/({1-m_0})]$. For every $\lambda\in I_{m_0}$, Forecaster sells a bet $\lambda$ that pays back $(1-\lambda(X-m_0))$ dollars when the data $X$ is revealed, for every dollar placed on $\lambda$ (before observing the data). It is well understood that bets of this form are the only set of admissible bets against the composite set of bounded null distributions with mean $m_0$~\citep{larsson2025variables}. Note that any such bet is fair from Forecaster's viewpoint, because the expected value of $(1-\lambda(X-m_0))$ (the amount Forecaster pays up per dollar of investment in $\lambda$) equals one if $X\sim P$ for some $P\in\mathcal P[0,1]$ with mean $m_0$. 

Starting with a unit wealth (i.e., $W_0 = 1$),  Skeptic bets against  Forecaster, hoping to get rich. The betting game between  Forecaster and  Skeptic can be described as follows. Let $\mathcal P(I_{m_0})$ denote the collection of all probability measures with support in the interval $I_{m_0}$. For $n = 1, 2, \dots$: 
    \begin{itemize}
    \item Forecaster makes the aforementioned bets available. 
    \item For some $\pi_n \in \mathcal P(I_{m_0})$, Skeptic invests a fraction $\pi_n(\lambda)$ of the current wealth, i.e.,  $W_{n-1} \pi_n(\lambda)$ dollars, in the bet $\lambda$, for every $\lambda \in I_{m_0}$. 
    \item Reality reveals the data $X_n$.
    \item Forecaster pays back $(1-\lambda(X_n - m_0))$ dollars to Skeptic for every dollar invested in the bet $\lambda$. 
    \item The wealth of Skeptic at time $n$ thus becomes $W_n = W_{n-1} \int_{I_{m_0}} \pi_n(\lambda) (1-\lambda(X_n-m_0)) d\lambda$. 
\end{itemize}

The goal of Skeptic is to choose a betting strategy $(\pi_n)_{n\ge 1}$ whose wealth is close to that of the best fixed betting strategy in hindsight. Formally, the performance of a given betting strategy can be evaluated by considering the difference between its log-wealth and that generated by using the best-in-hindsight strategy. This difference, termed as regret till time $n$, is given by 
\begin{align*} 
R_n &:= \ln \lrp{ \max\limits_{\lambda\in I_{m_0}}~  \prod\limits_{i=1}^n (1-\lambda(X_i - m_0))}  -  \ln W_n =\max\limits_{\lambda\in I_{m_0}}~ \sum\limits_{i=1}^n \ln\lrp{1-\lambda(X_i - m_0)}  -  \ln W_n. 
\end{align*}
The question we study is whether there is a betting strategy that can make $R_n$ above grow only as $O(\ln\ln n)$ (since achieving a $\ln n$ regret is straightforward, see Section~\ref{sec:logregret}).


In this work, we found it easier to directly work with the expression for the wealth process $W_n$ below than to explicitly specify and manipulate the per-round bets $\pi_n$, which are left implicit:
\[ 
W_n = \int\limits_{I_{m_0}} \prod\limits_{i=1}^n (1-\lambda(X_i - m_0)) \pi(\lambda) d\lambda. 
\]

\begin{remark}
    The set of bets $\lambda$ made available by the Forecaster, i.e., $I_{m_0}$, is precisely those that make the 1-round wealth $(1-\lambda(X - m_0))$ non-negative, for all $X\in [0,1]$. Thus, the Forecaster's belief that the distribution is bounded in $[0,1]$ with mean $m_0 \in (0,1)$, is without loss of generality, because if instead, the belief is that the distribution is bounded in $[-a,b]$ for some $0 < a$ and $0 < b$, and has mean $m_0 \in (-a, b)$, then he sells bets $\lambda$ for each $\lambda\in [-\tfrac{1}{a+m_0},\tfrac{1}{b-m_0}]$, and as earlier, pays back $(1-\lambda (X-m_0))$ dollars for each dollar invested in bet $\lambda$, when data $X$ is revealed. 
\end{remark}

\section{\cite{orabona2023tight}'s Wealth Process and its Regret Bound}\label{app:OJ}
In their work, \cite{orabona2023tight} use the following mixture wealth process: 
\[ W^{\operatorname{OJ}}_n := \int\limits_{-1}^1 W_n(\lambda) \pi^{\operatorname{OJ}}(\lambda) d\lambda,  \qquad \text{where} \qquad \pi^{\operatorname{OJ}}(\lambda) = \frac{\ln \ln (6.6 e)}{2|\lambda| \ln\lrp{\frac{6.6 e}{ |\lambda|}} \lrp{ \ln \ln\lrp{\frac{6.6 e}{|\lambda|}} }^2 }. \numberthis\label{eq:mixwealthbdd_OJ} \]

Further, let $\lambda^{*,\operatorname{OJ}}_n$ denote the hindsight-optimal bet in the restricted interval $[-1,1]$ that maximizes the wealth, i.e.,
\[ \lambda^{*,{\operatorname{OJ}}}_n \in \argmax\limits_{\lambda\in [-1,1]}~ \prod\limits_{i=1}^n \lrp{1 - \lambda (X_i - m_0)} \quad \text{ and }\quad W^{*,{\operatorname{OJ}}}_n = W_n(\lambda^{*,{\operatorname{OJ}}}_n). \]

\subsection{\cite{orabona2023tight}'s regret with respect to a restricted comparator class}
\citet{orabona2023tight} define regret of their mixture wealth process with respect to the wealth of the best-in-hindsight bet restricted to the sub-interval $[-1,1]$. In this section, we present an explicit bound on this modified regret, defined next: \[R^{\operatorname{OJ}}_n:= \ln W^{*,\operatorname{OJ}}_n - \ln W^{\operatorname{OJ}}_n = \max\limits_{\lambda \in [-1,1]}~\sum\limits_{i=1}^n \ln(1-\lambda(X_i - m_0)) - \ln W^{\operatorname{OJ}}_n.\] 

Recall $(S_n, V_n)$ from  discussion around~\eqref{eq:mixV}, and define $\calE_\alpha := \{ \sup_{n\ge 1} ~\ln W^{\operatorname{OJ}}_n \le \ln {1}/{\alpha} \}.$
As noted earlier, \cite{orabona2023tight} focus on deriving a confidence sequence for the mean of a bounded distribution that has an asymptotic width of $O(\sqrt{V_n\ln \ln V_n}/n)$. Their proof proceeds by establishing an implicit bound on the modified regret $R^{\operatorname{OJ}}_n$ defined above. We make this bound explicit below in~\eqref{eq:pathwiseregretrobbins_OJ}, and also show that it is $O(\ln\ln V_n)$ on all paths in the set $\mathcal E_\alpha$, a set of measure at least $1-\alpha$ under appropriate stochasticity assumption. As always, $R_n = 0$ and $S_n = 0$ whenever $V_n = 0$.

\begin{theorem}\label{th:bddregretbound_OJ}
    For all $n\ge 1$ such that $V_n > 0$, 
    \[   
        R^{\operatorname{OJ}}_n  \le 
        \begin{cases}
        6 + \ln\lrp{\frac{2}{\ln\ln(6.6e)}} + \ln\ln\lrp{14e \sqrt{1+V_n}} \\
        \quad+ 2 \ln\ln\ln\lrp{14e \sqrt{1+V_n}}, & \text{if } |S_n| < \sqrt{2V_n} ~\&~ |\lambda^{*,{\operatorname{OJ}}}_n| < 1 \\
        2+\ln\frac{1}{\pi^{\operatorname{OJ}}(1)}, &\text{if } |S_n| < \sqrt{2V_n}~\&~ |\lambda^{*,{\operatorname{OJ}}}_n| = 1\\
        \ln\frac{20\sqrt{e}}{3} + \ln\lrp{\frac{2}{\ln\ln (6.6 e)}} + \ln\!\ln\!\lrp{{14 e}\lrp{1+\sqrt{V_n}}} \\ 
        \quad + \ln\lrp{\!\frac{|S_n|}{\sqrt{\frac{4}{3}|S_n| + 2V_n}}\!} + 2\ln\!\ln\!\ln\!\lrp{{14 e}\lrp{1+\sqrt{V_n}}}, &\text{if } \sqrt{2V_n} \le |S_n| \le \frac{V_n}{5}\\
        \frac{1}{2} \ln W^{*,{\operatorname{OJ}}}_n - \ln \frac{\pi^{\operatorname{OJ}}(1)}{7}, &\text{if } \sqrt{2V_n} \le |S_n|~\&~ \frac{V_n}{5} < |S_n|.
        \end{cases}\numberthis\label{eq:pathwiseregretrobbins_OJ}
    \]
    Moreover, on $\calE_\alpha$, there exist constants $C_1 > 0$, $C_2 > 0$ and $K_\alpha$, such that 
    \[ \forall n\ge 1, \quad R^{\operatorname{OJ}}_n \le K_\alpha + C_1 \ln\ln\lrp{14e(1+\sqrt{V_n})} + C_2 \ln\ln\ln\lrp{14e(1+\sqrt{V_n})}. \numberthis\label{eq:LILOJ_OJregret} \]
    Furthermore, if the data are drawn from a distribution $P$ such that $\{W^{\operatorname{OJ}}_n\}_{n\ge 1}$ is a non-negative supermartingale, then $P[\calE_\alpha] \ge 1-\alpha$. 
\end{theorem}
Since, in this case, the prior is symmetric around $0$, and otherwise has similar structural properties like radial monotonicity as that for the prior in~\eqref{eq:mixwealthbdd},  we do not need to handle the two cases ($\lambda > 0$ and $\lambda < 0$) separately in the proof for the above theorem. In fact, the proof of the above theorem follows along the lines of that of Theorem~\ref{th:bddregretbound}. Hence, for brevity, we omit it from this paper. 

\subsection{\cite{orabona2023tight}'s regret with respect to the full comparator class}\label{app:OJRegretFull}

In this section, we bound the regret of $W^{\operatorname{OJ}}_n$, i.e., 
\[ R_n := \ln W^*_n - \ln W^{\operatorname{OJ}}_n = \max\limits_{\lambda \in I_{m_0}} \sum\limits_{i=1}^n \ln (1-\lambda(X_i - m_0))  - \ln W^{\operatorname{OJ}}_n,  \]
where, recall that $I_{m_0}:= [-\tfrac{1}{m_0}, \tfrac{1}{1-m_0}]$. 

Next, for $\lambda \in I_{m_0}$, recall that $f_n(\lambda)= \sum_{i=1}^n \ln(1-\lambda (X_i - m_0))$ with $f_n(0) = 0$, $\ln W^*_n = \max_{\lambda\in I_{m_0}} f_n(\lambda)$, and $\ln W^{*,\operatorname{OJ}}_n = \max_{\lambda \in [-1,1]} f_n(\lambda)$. 
\begin{lemma}
  For $\lambda \in I_{m_0}$ and $n \ge 1$, 
  \[ f_n(\lambda) \le \frac{1}{\beta_l} \ln W^{*,\operatorname{OJ}}_n, \quad \text{where} \quad \beta_l := \min\{ m_0, 1-m_0  \}. \]
\end{lemma}
\begin{proof}
  The inequality holds trivially for $\lambda \in [-1,1]$. We prove the other two cases ($\lambda < -1$ and $\lambda > 1$) separately. Next, since the right-hand side in the bound is non-negative, the inequality holds trivially whenever $f_n(\lambda) \le 0$. Thus, in the following, we focus on $\lambda \in I_{m_0}\setminus [-1,1]$ such that $f_n(\lambda) \ge 0$. 

  \noindent\textbf{Case 1 ($\lambda < -1$ and $f_n(\lambda) \ge 0$).} Since $\lambda \in I_{m_0}$, in this case, $ -\tfrac{1}{m_0} \le \lambda < -1$. Further, 
  \[ - 1 = \frac{1}{|\lambda|}\cdot \lambda + \lrp{ 1-\frac{1}{|\lambda|} } \cdot 0  \]
  Using the above and concavity of $f_n(\cdot)$, we get
  \[ f_n(-1) \ge \frac{1}{|\lambda|} \cdot f_n(\lambda) + 0 \ge m_0 f_n(\lambda),  \]
  which further gives
  \[ f_n(\lambda) \le \frac{f_n(-1)}{m_0}, \quad \text{ for all } \quad \lambda \in \lrs{-\frac{1}{m_0}, -1},  \]
  and hence,
  \[ f_n(\lambda) \le \frac{\ln W^{*,\operatorname{OJ}}_n}{\beta_l}, \quad \text{ for all } \quad \lambda \in \lrs{-\frac{1}{m_0}, -1}. \]
  
  \noindent\textbf{Case 2 ($\lambda > 1$ and $f_n(\lambda) \ge 0$). } As in the previous case, in this case, using concavity of $f_n(\cdot)$, we can show the following: 
  \[ f_n(\lambda) \le \frac{f_n(1)}{1-m_0}, \quad \text{ for all } \quad \lambda \in \lrs{ 1, \frac{1}{1-m_0}  }, \]
  and hence
  \[ f_n(\lambda) \le \frac{\ln W^{*,\operatorname{OJ}}_n}{\beta_l}, \quad \text{ for all } \quad \lambda \in \lrs{ 1, \frac{1}{1-m_0}  }, \]
  completing the proof.  
\end{proof}

Since the inequality in the lemma above holds for every $\lambda\in I_{m_0}$, the following corollary immediately follows by optimizing over $\lambda$.
\begin{corollary}\label{cor:WnstarWRTWnstarOJ}
  The following holds for all $n\ge 1$:
  \[ \ln W^*_n \le \frac{1}{\beta_l} \ln W^{*,\operatorname{OJ}}_n.  \]
\end{corollary}

Recall the set 
\[ \calE_\alpha = \lrset{ \sup\limits_{n\ge 1}~ \ln W^{\operatorname{OJ}}_n  \le \ln \tfrac{1}{\alpha} }.  \]

\begin{proposition}\label{prop:lilregretoje_alpha}
  For all $n\ge 1$, 
  \[ R_n \le \frac{1}{\beta_l} R^{\operatorname{OJ}}_n + \lrp{\frac{1}{\beta_l} - 1} \ln W^{\operatorname{OJ}}_n,  \numberthis\label{eq:R_nOJbound} \]
  where $R^{\operatorname{OJ}}_n := \ln W^{*,\operatorname{OJ}}_n - \ln W^{\operatorname{OJ}}_n$ is the regret of the log wealth process $\ln W^{\operatorname{OJ}}_n$ with respect to the best in the restricted comparator class $[-1,1]$. Moreover, on $\calE_\alpha$, there exist constants $C_1 > 0$, $C_2 > 0$ and $K_\alpha$, such that 
  \[ \forall n\ge 1, \quad R_n \le K_\alpha + C_1 \ln\ln\lrp{14e(1+\sqrt{V_n})} + C_2 \ln\ln\ln\lrp{14e(1+\sqrt{V_n})}. \numberthis\label{eq:LILOJregret} \]
  Furthermore, if the data are drawn from a distribution $P$ such that $\{W^{\operatorname{OJ}}_n\}_{n\ge 1}$ is a non-negative supermartingale, then $P[\calE_\alpha] \ge 1-\alpha$. 
\end{proposition} 
\begin{proof}
  Consider the following inequalities:
  \begin{align*}
    R_n 
    &= \ln W^*_n - \ln W^{\operatorname{OJ}}_n\\
    &\le \frac{1}{\beta_l} \ln W^{*,\operatorname{OJ}}_n - \ln W^{\operatorname{OJ}}_n  \tag{Corollary~\ref{cor:WnstarWRTWnstarOJ}} \\
    &= \frac{1}{\beta_l}\lrp{ \ln W^{\operatorname{OJ}}_n + R^{\operatorname{OJ}}_n  } - \ln W^{\operatorname{OJ}}_n, 
  \end{align*}
  proving the bound in~\eqref{eq:R_nOJbound}. 

  Since on $\calE_\alpha$, $\ln W^{\operatorname{OJ}}_n \le \ln\tfrac{1}{\alpha}$, using this in~\eqref{eq:R_nOJbound} we get that on $\calE_\alpha$, 
  \[\forall n\ge 1, \qquad R_n \le \frac{1}{\beta_l} R^{\operatorname{OJ}}_n + \lrp{\frac{1}{\beta_l} - 1} \ln \frac{1}{\alpha}.  \]
  The inequality in~\eqref{eq:LILOJregret} then follows from combining the above with~\eqref{eq:LILOJ_OJregret} from Theorem~\ref{th:bddregretbound_OJ}, completing the proof.
\end{proof}

\begin{remark}\label{rem:linearOJregret}
    We note that while the regret of the wealth process $W^{\operatorname{OJ}}_n$ is $O(\ln \ln V_n)$ on the set $\calE_\alpha$ (Proposition~\ref{prop:lilregretoje_alpha}), it can be linear on the complement.  This is because the restricted Robbins’ mixture has no mass outside $[-1,1]$. Therefore, when the data are, say, iid from a distribution with mean $m$ that is far from $m_0$ (and hence, the optimal bet is outside this interval), the mixture lacks weight to compete with the best-in-hindsight, and hence, suffers linear regret in those cases. This linear regret affects the growth rate of this wealth process, leading to a smaller growth rate compared to that of the uniform-mixture or modified Robbins' wealth processes from Sections~\ref{sec:logregret} and~\ref{sec:regretbound}, respectively. The example below demonstrates this. \\
    Consider the deterministic sequence $X_i = 1$ for all $i \in [n]$. Then, $f_n(\lambda)=n \ln(1-\lambda(1-m_0))$, $\ln W^*_n = f_n(-{1}/{m_0}) = n \ln({1}/{m_0})$, and $\ln W^{\operatorname{OJ}}_n \le \ln W^{*,\operatorname{OJ}}_n = n \ln(2-m_0)$. Hence, $$R_n = \ln W^*_n - \ln W^{\operatorname{OJ}}_n \ge \ln W^*_n - \ln W^{*,\operatorname{OJ}}_n = -n \ln\lrp{{m_0(2-m_0)}}.$$ 
    Next, since $W^{\operatorname{OJ}}_n \le W^{*,\operatorname{OJ}}_n$, we also have 
    \[\liminf_n \frac{\ln W^{\operatorname{OJ}}_n}{n} \le \liminf_n \frac{\ln W^{*,\operatorname{OJ}}_n}{n} = \liminf_n \lrp{\frac{\ln W^*_n + n\ln(m_0(2-m_0))}{n}} < \ln \frac{1}{m_0}, \]
    where the right-hand side is also the growth rate of the uniform-mixture as well as the modified Robbins' mixture wealth processes.
\end{remark}

\section{Does NSM imply NM?}\label{app:supermartingale}
Let $X_i \in [0,1]$ for all $i\in [n]$,  $\lambda_1 > 0$, $\lambda_2 < 0$ and $\lambda_1, \lambda_2 \in I_{m_0}$. For a fixed $\pi \in (0,1)$, suppose the mixture
\[
W_n := \pi W_n(\lambda_1) + (1 - \pi) W_n(\lambda_2)
\]
is a nonnegative supermartingale (NSM) with respect to a class of distributions $\mathcal{P}$ on $[0,1]^\infty$. We ask whether this implies that $W_n$ is in fact a nonnegative martingale (NM) for this class, and thus whether the NSM constraint automatically forces
$\mathcal{P}$ to consist only of distributions with conditional mean $m_0$. The answer is negative. We will show this by starting with 
$$
\mathcal P = \{ \mathbf P: \Exp{\mathbf P}{X_n \mid \mathcal F_{n-1}}=m_0 \text{ for all } n \geq 1 \},
$$ 
for which $(W_n)_{n\ge 1}$ is a NM,
and constructing another distribution that we can add to this set, for which $(W_n)_{n\ge 1}$ is only an NSM and not an NM. In this appendix, we use the bolded notation $\mathbf P$ to denote distributions on $[0,1]^\infty$, to distinguish them from distributions $P$ on $[0,1]$.

To proceed, first note that the supermartingale property implies 
\[
\mathbb{E}_{\mathbf P}[W_n \mid \mathcal{F}_{n-1}] \le W_{n-1}, \quad \mathbf P\text{-almost surely, } \forall \mathbf P  \in \mathcal P.
\]
Expanding using the definition of $W_n$, and on rearranging, the above condition becomes
\begin{equation}\label{eq:supermg-condition}
A_{n-1} \cdot \mathbb{E}_{\mathbf P}[X_n - m_0 \mid \mathcal{F}_{n-1}] \ge 0 \quad \mathbf P\text{-almost surely, }\forall \mathbf P\in\mathcal P,
\end{equation}
where $A_{n-1} := \pi W_{n-1}(\lambda_1)\lambda_1 + (1 - \pi) W_{n-1}(\lambda_2)\lambda_2$, which is $\mathcal{F}_{n-1}$-measurable. 

For $\delta\in(0,\min\{m_0,1-m_0\})$, define the probability law $\mathbf P_\delta$ on $\{0,1\}^\infty$ recursively by
\[
\mathbf P_\delta(X_n=1\mid \mathcal F_{n-1})= m_0 + \delta \mathrm{sign}(A_{n-1}),
\]
where we let $\mathrm{sign}(0)=0$.  Equivalently,
\[
\Exp{\mathbf P_\delta}{X_n-m_0\mid \mathcal F_{n-1}}
=
\delta\mathrm{sign}(A_{n-1}).
\]
Then
\[
A_{n-1}\cdot\Exp{\mathbf P_\delta}{X_n-m_0\mid \mathcal F_{n-1}}
=
\delta |A_{n-1}| \ge 0
\qquad \mathbf P_\delta\text{-a.s.},
\]
as required by condition~\eqref{eq:supermg-condition}.
Hence $
\Exp{\mathbf P_\delta}{W_n\mid \mathcal F_{n-1}} \le W_{n-1}$, so $(W_n)_{n\ge 0}$ is a NSM under $\mathbf P_\delta$, for every $\delta \in (0,\min\{m_0, 1-m_0\})$. However, it is not a martingale, because whenever $\mathbf P_\delta(A_{n-1}\neq 0)>0$,
\[
A_{n-1}\cdot\Exp{\mathbf P_\delta}{X_n-m_0\mid \mathcal F_{n-1}}
=
\delta |A_{n-1}|>0
\]
on a set of positive probability. 

To conclude, $W_n$ is a NM for $\mathcal P$ but only a NSM for $\mathcal P \cup  \{\mathbf P_{\delta}\}_{\delta \in (0,\min\{m_0, 1-m_0\})}$.

An analogous argument also works when $W_n$ is not just a mixture over two constants $\lambda_1,\lambda_2$ of opposite sign, but continuous mixtures as considered in this work; only the definition of $A_n$ needs to be amended.

\section{Summary of Regret and Growth Guarantees}\label{app:table}
\begin{table}[ht]
\centering
\small
\setlength{\tabcolsep}{4pt}
\renewcommand{\arraystretch}{1.3}
\newcommand{\RC}{\raggedright}
\newcommand{\rowgap}{\addlinespace[5pt]}
\begin{tabular}{@{}p{2.7cm}p{2.7cm}p{2.7cm}p{3.4cm}p{2.2cm}@{}}
\toprule
\RC\textbf{Strategy} & \RC\textbf{Prior} & \RC\textbf{Worst-case regret (every path)} & \RC\textbf{Typical-path regret} & \RC\textbf{Growth rate (iid alternatives)} \tabularnewline
\midrule
\midrule
\multicolumn{5}{c}{\textbf{Bounded data}: $X_t\in[0,1]$; comparator class $I_{m_0}$} \tabularnewline
\midrule
\midrule
\RC Uniform (Cover) mixture
  & \RC uniform on $I_{m_0}$
  & \RC $O(\ln n)$ \newline (Thm.~\ref{th:ville_gaussianmixture})
  & \RC no improvement: $\ge \tfrac12\ln V_n - O(1)$ \newline (Lem.~\ref{lem:cover-lower})
  & \RC optimal \newline (Eq.~\eqref{eq:loggrowthrate}) \tabularnewline
  \rowgap
\RC Restricted Robbins' mixture \citep{orabona2023tight}
  & \RC Robbins prior on $[-1,1]\subsetneq I_{m_0}$
  & \RC linear on some paths \newline (Rem.~\ref{rem:linearOJregret})
  & \RC $O(\ln\ln V_n)$ on $\calE_\alpha$ \newline (Prop.~\ref{prop:lilregretoje_alpha})
  & \RC suboptimal for some $Q$ \newline (Rem.~\ref{rem:linearOJregret}) \tabularnewline
  \rowgap
  \rowgap
\RC\textbf{Modified Robbins' mixture (ours)}
  & \RC modified Robbins prior on $I_{m_0}$ \newline (Eq.~\eqref{eq:mixwealthbdd})
  & \RC $O(\ln n)$ \newline (Rem.~\ref{rem:lognRobbinsregret})
  & \RC $O(\ln\ln V_n)$ on $\calE_\alpha$ for all $n$ (Thm.~\ref{th:bddregretbound}); \newline $(1+o(1))\ln\ln V_n$ eventually on $\widetilde{\calE}_0$ (Thm.~\ref{th:lilasregretV})
  & \RC optimal \newline (Eq.~\eqref{eq:modrobbinsgrowthrate}) \tabularnewline
  \rowgap
\RC \textbf{Hedge: uniform $+$ restricted Robbins (ours)}
  & \RC $s_0\cdot$uniform $+\,(1-s_0)\cdot$restricted Robbins
  & \RC $O(\ln n)$ \newline (Prop.~\ref{prop:BOB})
  & \RC $O(\ln\ln V_n)$ on $\calE_\alpha$ \newline (Prop.~\ref{prop:BOB})
  & \RC optimal \newline (Prop.~\ref{prop:BOB}) \tabularnewline
\addlinespace[8pt]
\midrule
\midrule
\multicolumn{5}{c}{\textbf{Sub-Gaussian data} \citep{agrawal2025regret}: comparator class $\mathbb R$} \tabularnewline
\midrule
\midrule
\RC Gaussian mixture
  & \RC Gaussian on $\mathbb R$
  & \RC can be linear
  & \RC $O(\ln \Bar{V}_n)$
  & \RC optimal \tabularnewline
  \rowgap
\RC Robbins' mixture (restricted)
  & \RC Robbins-type, heavy near $0$
  & \RC can be linear
  & \RC $O(\ln\ln \Bar{V}_n)$
  & \RC suboptimal for some $Q$ \tabularnewline
  \rowgap
\RC \textbf{Hedge: Gaussian $+$ Robbins (ours)}
  & \RC $s_0\cdot$Gaussian $+\,(1-s_0)\cdot$Robbins
  & \RC can be linear
  & \RC $O(\ln\ln \Bar{V}_n)$ \newline (Prop.~\ref{prop:BOB})
  & \RC optimal \newline (Prop.~\ref{prop:BOB}) \tabularnewline
\bottomrule
\end{tabular}
\caption{Regret and growth guarantees of the strategies discussed in this paper. Regret is measured against the best fixed bet in the stated comparator class. Typical-path bounds for bounded data hold path-wise: on the Ville event $\calE_\alpha$, which has probability at least $1-\alpha$ under the conditional-mean claim, and eventually on $\widetilde{\calE}_0=\{V_n\uparrow\infty,\ \sup_n W_n<\infty\}$, which has probability one under the conditional-mean claim whenever $V_n\to\infty$ almost surely. For sub-Gaussian data, typical paths are similarly defined in \citet{agrawal2025regret}, and $\Bar{V}_n$ is a certain process measuring the intrinsic variance. ``Optimal'' growth is the best achievable rate in the respective setting. For bounded and iid data from $Q$ with mean $m\neq m_0$, it equals $\KLinf(Q,m_0)$. Relative to its own comparator class $[-1,1]$, the restricted Robbins' mixture has $O(\ln\ln V_n)$ regret on $\calE_\alpha$ (Thm.~\ref{th:bddregretbound_OJ}).}
\label{tab:comparison}
\end{table}

\end{document}